\documentclass[a4paper,11pt]{article}
\usepackage{fullpage} 

\usepackage{diagbox} 
\usepackage{slashbox} 
\usepackage{tabularx} 
\usepackage{bbm} 
\usepackage{pifont}  

\usepackage{amsmath}
\usepackage{amssymb}
\usepackage{amsthm}
\usepackage{mathtools}
\usepackage{mathdots}
\usepackage{appendix}
\usepackage{graphicx}
\usepackage{enumerate}
\usepackage{microtype}
\usepackage{mleftright}
\usepackage{mdframed}
\usepackage{authblk}
\usepackage{tcolorbox}
\usepackage[english]{babel}
\usepackage[utf8]{inputenc}
\usepackage{algorithm}
\usepackage[noend]{algpseudocode}
\usepackage{xspace}
\usepackage{tikz}
\usepackage{pgfplots}
\pgfplotsset{width=8cm,compat=1.9}
\usepackage{mathdots}
\usepackage{url}
\usepackage{forest}
\forestset{default preamble={for tree={s sep+=1cm}}}
\usepackage{caption,subcaption}
\usepackage[noadjust]{cite}
\usepackage{hyperref}

\tcbset{
    rounded corners,
    colback = white,
    before skip = 0.2cm,
    after skip = 0.5cm,
    boxrule = 1.5pt,
    arc = 5pt
}

\makeatletter
\renewcommand{\ALG@name}{Protocol}
\makeatother
\allowdisplaybreaks

\graphicspath{ {./figures/} }

\usepackage{algorithm}
\usepackage[noend]{algpseudocode}

\DeclarePairedDelimiter{\ceil}{\lceil}{\rceil}
\DeclarePairedDelimiter{\floor}{\lfloor}{\rfloor}

\newtheorem{theorem}{Theorem}[section]
\newtheorem*{theorem*}{Theorem}
\newtheorem{proposition}[theorem]{Proposition}
\newtheorem{lemma}[theorem]{Lemma}
\newtheorem{corollary}[theorem]{Corollary}
\newtheorem*{corollary*}{Corollary}

\newtheorem{assumption}[theorem]{Assumption}

\theoremstyle{definition}
\newtheorem{definition}[theorem]{Definition}

\newcommand{\cH}{\mathcal{H}}
\newcommand{\cX}{\mathcal{X}}
\newcommand{\cY}{\mathcal{Y}}

\newcommand{\cF}{\mathcal{F}}
\newcommand{\cR}{\mathcal{R}}

\newcommand{\cB}{\mathcal{B}}

\newcommand{\cT}{\mathcal{T}}

\newcommand{\N}{\mathbb{N}}

\newcommand{\ind}{\mathbbm{1}}
\newcommand{\VC}{\operatorname{VC}}

\newcommand{\Ex}{\mathop{\mathbb{E}}}

\newcommand{\cA}{\mathcal{A}}

\newcommand{\Ber}{\mathrm{Ber}}
\newcommand{\Bin}{\mathrm{Bin}}

\newcommand{\Lrn}{\mathsf{Lrn}}
\newcommand{\Adv}{\mathsf{Adv}}

\newcommand{\eps}{\varepsilon}

\newcommand{\etoe}{\mathsf{e2e}}
\newcommand{\ct}{\mathsf{CoT}}

\newcommand{\LD}{\mathtt{L}}

\newcommand{\tree}{\mathbf{T}}
\newcommand{\iter}{M}
\newcommand{\rounds}{T}
\newcommand{\lin}{\operatorname{lin}}

\DeclareMathOperator*{\argmax}{arg\,max}

\newif\ifanonymous
\anonymousfalse

\begin{document}
\title{A Theory of Online Learning with Autoregressive Chain-of-Thought Reasoning}

\author[1]{Ilan Doron-Arad}
\author[2]{Idan Mehalel}
\author[1]{Elchanan Mossel}
\affil[1]{MIT}
\affil[2]{The Hebrew University}


\maketitle

\begin{abstract}
Autoregressive generation lies at the heart of the mechanism of large language models. It can be viewed as the repeated application of a next-token generator: 
starting from an input string (prompt), the generator is applied for \(\iter\) steps, and the last generated token is taken as the final output. 
\cite{joshi2025theory} proposed a PAC model for studying the learnability of the input-output maps arising from this process. 
We develop an online analogue of this framework, focusing on the mistake bound of learning the final output induced by an unknown next-token generator.

We distinguish between two forms of feedback. 
In the End-to-End model, after each round the learner observes only the final token produced after \(\iter\) autoregressive steps. 
In the Chain-of-Thought model, the learner is additionally shown the entire \(\iter\)-step trajectory. 
Our goal is to understand how the optimal mistake bound depends on the generation horizon \(\iter\), and to what extent observing intermediate tokens can reduce this dependence.

Our main results show that the online theory of autoregressive learning exhibits a qualitative picture analogous to the statistical one found by \cite{hanneke2026sample}, but with a different scale of dependence on the generation horizon.
In the End-to-End model, we prove a taxonomy of possible mistake-bound growth rates in the generation horizon \(\iter\): subject to mild regularity conditions, every rate between constant and logarithmic can arise. 
We also show that this logarithmic ceiling is unavoidable for the online setting, in the sense that every class of finite Littlestone dimension exhibits at most logarithmic dependence on \(\iter\). 
In the Chain-of-Thought model, we show that access to the full generated trajectory eliminates the dependence on \(\iter\) altogether, mirroring the picture proved in the statistical setting.

We also analyze autoregressive linear threshold classes. 
For autoregressive linear thresholds in \(\mathbb{R}^d\), we prove that the optimal mistake bound is \(\Theta(d^2)\) under both End-to-End and Chain-of-Thought feedback, and for any generation length $\iter$. 
The methods developed for this analysis also yield new lower bounds in the statistical setting.

Along the way, our results resolve several questions left open by \cite{joshi2025theory}. 
In particular, we show that even for classes of finite Littlestone dimension, the End-to-End statistical sample complexity can depend on the generation length \(\iter\).
\end{abstract}

\pagebreak
\setcounter{tocdepth}{2}
\tableofcontents
\pagebreak

\section{Introduction}

Modern large language models are not learned once from a fixed collection of independent examples and then used in isolation. They are trained on data accumulated over time, updated through successive rounds of optimization, deployed in interactive environments, and increasingly adapted to the particular users and contexts in which they operate. In such settings, the classical picture of learning from an i.i.d sample captures only part of the story. The learner observes a stream of prompts, receives feedback, changes its behavior, and is then evaluated on future interactions. This is precisely the kind of situation for which online learning was designed.

At the same time, the objects being learned are fundamentally autoregressive. A language model does not usually produce an answer in a single atomic step. Rather, starting from a prompt, it generates one token at a time; each newly generated token becomes part of the context for the next prediction. After many such steps, the final token, sentence, or answer is interpreted as the model’s output. Thus, even when the desired task is an end-to-end prediction problem, the mechanism producing the prediction is an iterated next-token process.

This creates a natural distinction between two forms of supervision. In one regime, the learner observes only the final answer produced after the autoregressive process has run for $\iter$ steps. This corresponds to \emph{end-to-end supervision}: the intermediate reasoning, if any, is hidden. In a stronger regime, the learner also observes the entire generated trajectory. This corresponds to \emph{chain-of-thought supervision}: the learner sees not only what answer was produced, but also the sequence of intermediate tokens through which the generator arrived there. Intuitively, such intermediate information should be useful, as it exposes more of the underlying next-token generation mechanism.

A recent line of work initiated by \cite{joshi2025theory} introduced a formal framework for studying this phenomenon in the PAC setting. In their model, the primitive object is a next-token generator $f: \Sigma^\star \to \Sigma$, where $\Sigma$ is the alphabet of possible tokens. Iterating $f$ for $\iter$ steps from an input string produces an $\iter$-token autoregressive trajectory, whose last token is the end-to-end output. This framework makes it possible to ask, in precise learning-theoretic terms, how the difficulty of learning the induced input-output map depends on the generation length $\iter$, and how much chain-of-thought supervision can reduce this dependence. These two fundamental questions were studied in the PAC-learning framework of \cite{joshi2025theory} by \cite{hanneke2026sample}.

The present paper develops the online analogue of this theory. Rather than receiving an i.i.d training sample, the learner interacts with an adversarially chosen sequence of examples. On each round, the learner is given a prompt, predicts the final token generated after $\iter$ autoregressive steps, and then receives feedback. In the end-to-end model, the feedback consists only of the final token. In the chain-of-thought model, the feedback consists of the full $\iter$-step trajectory. The central quantity in this setting is no longer sample complexity, but the optimal mistake bound: how many prediction errors are unavoidable before the learner has identified enough of the hidden generator to perform well on all subsequent rounds?

There is also a second, more structural reason to study the online model. Online learnability is often a stronger and more robust notion than ordinary distributional learnability. A class that can be learned in the online mistake-bound model is not merely learnable from an i.i.d. sample; it often remains learnable under additional constraints that make standard PAC learning more demanding. A prominent example is private PAC learning, which is equivalent to online learning, in the sense that a class is private PAC learnable if and only if it is online learnable \cite{alon2022private}.  Thus, online learnability serves as a powerful certificate of learnability: it indicates that the class has enough combinatorial structure to be learned even in settings where the data are sequential, adversarial, or subject to privacy constraints.

We now turn to the formal definition of the model in study. 


\subsection{The autoregressive online learning model} \label{sec:model}

We start by recalling the general autoregressive learning framework presented in \cite{joshi2025theory} and defined below. A \emph{next-token-generator} is a function
\[
f: \Sigma^\star \to \Sigma,
\]
where $\Sigma$ is some finite alphabet. In this paper, we focus on the binary case, so $\Sigma = \{0,1\}$ unless stated otherwise. The function $f$ naturally induces an \emph{autoregressive generation} process. Define the \emph{apply-and-append} map $\bar{f}: \Sigma^\star \to \Sigma^\star$  of $f$ by
\[
\bar{f}(x) := x \circ f(x).
\]
Denote
\[
\bar f^{\,\iter}
=
\underbrace{\bar f\circ\cdots\circ \bar f}_{\iter\text{ times}}.
\]
We define the \emph{$\iter$-step chain of thought} of $f$ on $x$ to be the length $\iter$ suffix of $\bar{f}^\iter(x)$, denoted by $f^{\ct-\iter}(x)$. The corresponding \emph{end-to-end output} of the $\iter$-step chain of thought of $f$ on $x$ is the last symbol of $f^{\ct-\iter}(x)$, denoted as $f^{\etoe-\iter}(x) := f^{\ct-\iter}(x)[-1]$.

So, each class $\cF$ of next-token-generators (also called a \emph{base class}) naturally induces two classes of interest:

\[
\mathcal F^{\mathsf{CoT}-\iter}
:=
\{f^{\mathsf{CoT}-\iter}:f\in\mathcal F\},
\qquad
\mathcal F^{\mathsf{e2e}-\iter}
:=
\{f^{\mathsf{e2e}-\iter}:f\in\mathcal F\}.
\]

In this paper, we are interested in learning the class $\cF^{\etoe-\iter}$ online in two regimes: $\etoe$-learning, also called end-to-end learning, and $\ct$-learning, also called end-to-end learning with chain-of-thought supervision. Let us start with the $\etoe$-learning regime. We consider a game played between a learner $\Lrn$ and an adversary $\Adv$ played for $\rounds$ many rounds. Before the game begins, $\Adv$ chooses a \emph{target function} $f_\star \in \cF$. In each round $t \in [\rounds]$, $\Adv$ first sends an instance $x_t \in \Sigma^\star$ to $\Lrn$, and then $\Lrn$ chooses a symbol (or a \emph{label}, in classification theory terminology) $\hat{y}_t \in \Sigma$. Finally, $\Adv$ reveals the correct end-to-end output $f_\star^{\etoe-\iter}(x_t)$, and $\Lrn$ suffers loss $1$ if and only if $\hat{y}_t \neq f_\star^{\etoe-\iter}(x_t)$. Most of our results hold already for a deterministic learner, so unless stated otherwise, we limit our study to the case that $\Lrn$ is deterministic. We now define the $\ct$-learning regime. $\ct$-learning is very similar to $\etoe$-learning; the only difference is that in $\ct$-learning, $\Adv$ reveals the entire chain-of-thought  $f_\star^{\ct-\iter}(x_t)$, instead of only $f_\star^{\etoe-\iter}(x_t)$. 

In both regimes, $\Lrn$'s goal is to minimize its total cumulative loss throughout the game, also called the \emph{mistake bound}. $\Adv$'s goal is to maximize this quantity. Note that $\etoe$-learning amounts to standard online learning of the class $\cF^{\etoe-\iter}$. In contrast, in $\ct$-learning $\Lrn$ receives $f_\star^{\ct-\iter}(x_t)$ as feedback from $\Adv$, but tested only on correctness of the final symbol of $f_\star^{\ct-\iter}(x_t)$, that is, on $f_\star^{\etoe-\iter}(x_t)$.

The main goal in this paper is to understand how the mistake bound of autoregressive learning scales with the generation length $\iter$, and to what extent chain-of-thought supervision may help reducing the mistake bound. Throughout the paper, we assume that $\iter \geq 2$, as the case $\iter = 1$ is equivalent to the standard online learning scenario.

\section{Results}
This section assumes some familiarity with basic terms in learning theory, such as PAC-learning. The unfamiliar reader may read Section~\ref{sec:prel}, where we present all relevant technical background in a self-contained manner.

\subsection{Background on the Littlestone dimension}
We first give some relevant background on the Littlestone dimension, which is necessary to understand the results.
The main combinatorial quantity that arises when analyzing mistake bounds of online learning a hypothesis class $\cH \subset \cY^{\cX}$ of functions from a \emph{domain} $\cX$ to a \emph{label space} $\cY$ is the \emph{Littlestone dimension} \cite{littlestone1988learning, daniely2015multiclass} of $\cH$, denoted as $\LD(\cH)$. Let us informally define the Littlestone dimension. A formal definition is given in Section~\ref{sec:prel-dimensions}. $\LD(\cH)$ is the maximal depth of a perfect \emph{Littlestone tree} that is shattered by $\cH$. A perfect Littlestone tree for $\cH$ is a rooted perfect binary tree whose internal vertices are labeled by instances from $\cX$, and for each internal vertex, its two outgoing edges are labeled by two different labels from $\cY$. Such a tree is \emph{shattered} by $\cH$ if every branch $b$ (root-to-leaf path) in the tree has a function $h_b\in \cH$ that agrees with $b$. Here, ``agrees" means that for all $x \in \cX$ that labels an internal vertex in $b$, if the outgoing edge of the internal vertex labeled by $x$ that lies in $b$ is labeled with $y \in \cY$, then $f_b(x) = y$.
A fundamental result of \cite{littlestone1988learning, daniely2015multiclass} states that the optimal mistake bound of learning $\cH$ is precisely $\LD(\cH)$.

Since $\etoe$-learning of $\cF$ with generation length $\iter$ amounts to standard online learning of the class $\cF^{\etoe-\iter}$, our analysis focuses on bounding $\LD(\cF^{\etoe-\iter})$. On the other hand, $\ct$-learning is not a standard online learning task: while the loss is measured with respect to the target function $f_\star^{\etoe-\iter} \in \cF^{\etoe-\iter}$, the feedback is given by the target function $f_\star^{\ct-\iter} \in \cF^{\ct-\iter}$. A useful approach towards proving upper bounds for $\ct$-learning exploited  in \cite{joshi2025theory, hanneke2026sample} is to replace $\ct$-learning in the more difficult task of learning the class $\cF^{\ct-\iter}$. This task is indeed harder, since in contrast with $\ct$-learning, any disagreement between the learner's full chain-of-thought prediction $\hat{y} \in \Sigma^{\iter}$ counts as a mistake, even if the final End-to-End output is correct. We use the same approach and derive mistake bounds for $\ct$-learning  by bounding $\LD(\cF^{\ct-\iter})$. Unless stated otherwise, we limit our discussion only for base classes with finite Littlestone dimension. This limitation is analogous to the limitation of finite VC-dimension in the statistical setting \cite{joshi2025theory, hanneke2026sample}. Further justification for this limitation is given in Section~\ref{sec:non-lit}, where we show that for non-Littlestone classes, learnability of $\cF^{\etoe-\iter}$ can alternate between the two extremes of impossible even with $\ct$-supervision, and trivial even without $\ct$-supervision. A similar result was proved in \cite{hanneke2026sample} for the statistical setting, for non-VC classes.

\subsection{Learning with chain-of-thought supervision}
The results in this section are proved in Section~\ref{sec:cot}
Our first main result shows that having $\LD(\cF) < \infty$ and chain-of-thought supervision eliminates any dependence on the generation length $\iter$.

\begin{theorem}\label{thm:cot-bound-lit-intro}
    For every base class $\cF \subset \Sigma^{\Sigma^\star}$ (even if $|\Sigma| > 2$, and even for $|\Sigma| = \infty$), and for every generation length $\iter$, we have
    \[
    \LD(\cF^{\ct-\iter})\leq \LD(\cF).
    \]
\end{theorem}
Notably, Theorem~\ref{thm:cot-bound-lit-intro} holds even when $|\Sigma| > 2$.
\cite{hanneke2026sample} proved an analog result for the statistical setting: having $\VC(\cF) < \infty$ (where $\VC(\cF)$ denoted the VC-dimension of $\cF$) and chain-of-thought supervision eliminates any dependence of the PAC-sample complexity on the generation length $\iter$.

It is not hard to think of autoregressively-degenerate classes (classes of functions that keeps generating the same bit forever) for which $\LD(\cF^{\ct-\iter})\geq \LD(\cF)$ for all $\iter$, and so the bound of Theorem~\ref{thm:cot-bound-lit-intro} cannot be improved in general, not even by a constant factor.

Theorem~\ref{thm:cot-bound-lit-intro} was also proved independently and concurrently in \cite{balcan2026online}.

Theorem~\ref{thm:cot-bound-lit-intro} can be used to improve, for some classes, the PAC $\ct$-learning sample complexity bound given by \cite{hanneke2026sample}. Their result shows that the PAC sample complexity of $\ct$-learning is proportional to $\VC(\cF) \cdot \VC^\star(\cF)$, where $\VC^\star(\cF) = \VC(\cF^\star)$, where $\cF^\star$ is the dual class of $\cF$, which is the class obtained by ``switching roles" between $\cX$ and $\cF$. However, $\VC^\star(\cF)$ can grow as $2^{\VC(\cF)}$. Using a standard online-to-batch conversion, we deduce the following sample complexity bound. Notably, it holds even when $|\Sigma| > 2 $.

\begin{theorem} \label{thm:pac-cot-bound-lit-intro}
    For every base class $\cF \subset \Sigma^{\Sigma^\star}$ (even if $|\Sigma| > 2$, and even for $|\Sigma| = \infty$), and for all $\iter$, if $\LD(\cF) < \infty$ the class $\cF^{\ct-\iter}$ is PAC-learnable with sample complexity
    \[
    m(\eps, \delta) = O \mleft( \frac{\LD(\cF) + \log (1/\delta)}{\eps} \mright).
    \]
\end{theorem}

So, the advantage of Theorem~\ref{thm:pac-cot-bound-lit-intro} over the result stated in \cite{hanneke2026sample} is that it holds for non-binary base classes $\cF$, and also that it provides a better bound for every binary class $\cF$ for which $\LD(\cF) < \VC(\cF) \VC^\star(\cF)$, such as $d$-hamming balls.

Along the way, Theorem~\ref{thm:pac-cot-bound-lit-intro} also settles a $\log \iter$ gap from \cite{joshi2025theory}, who proved essentially the same bound as of Theorem~\ref{thm:pac-cot-bound-lit-intro}, but with $\LD(\cF) \log \iter$ replacing the $\LD(\cF)$ in our bound. They asked if the logarithmic dependence on $\iter$ is necessary, and our bound shows that it is not.

We note here that Theorem~\ref{thm:pac-cot-bound-lit-intro} may also be deduced via \emph{stable sample compression schemes}, using a result of \cite{hanneke2026sample}. Concretely, they proved that if $\cF$ has a stable sample compression scheme of size $k$, then $\cF^{\ct-\iter}$ also has a stable sample compression scheme of size $k$. Since every class $\cF$ has a stable sample compression scheme of size $O(\LD(\cF))$ \cite{nikita2017optimal}, Theorem~\ref{thm:pac-cot-bound-lit-intro} is implied.

\subsection{End-to-End learning}
We first establish a general upper bound that bounds how $\LD(\cF^{\etoe-\iter})$ can scale with $\iter$ in the worst case.

\begin{theorem} \label{thm:e2e-upper-bound-intro}
    For every base class $\cF$, and for every generation length $\iter$, we have
    \[
    \LD(\cF^{\etoe-\iter}) = O( \LD(\cF) \log (\LD(\cF) \iter)).
    \]
\end{theorem}

Theorem~\ref{thm:e2e-upper-bound-intro} is proved in Section~\ref{sec:e2e-upper-bound}.
While it gives a strong guarantee for any base class $\cF$ (of finite Littlestone dimension), unlike the guarantee of Theorem~\ref{thm:cot-bound-lit-intro}, it is not clear how tight Theorem~\ref{thm:e2e-upper-bound-intro} is. Indeed, it is not hard to think of classes $\cF$ for which $\LD(\cF^{\etoe-\iter}) \leq \LD(\cF)$ for all $\iter$, and it is not immediately clear if there are classes for which $\LD(\cF^{\etoe-\iter})$ scales as $\LD(\cF) \log \iter$, or even as some sub-logarithmic monotone function of $\iter$. This naturally raises the following question:

\begin{tcolorbox}
\begin{center}
What are the possible growth rates of $\LD(\cF^{\etoe-\iter})$ in terms of $\LD(\cF)$ and $\iter$?
\end{center}
\end{tcolorbox}

This question was considered by \cite{hanneke2026sample} in the statistical setting. They proved that under mild regularity conditions, essentially any at-most-linear growth rate of $\VC(\cF^{\etoe-\iter})$ with $\iter$ is realized by some class $\cF$. We prove the analogous result for the online setting, with similar mild regularity conditions to those of \cite{hanneke2026sample}. Let us explicitly state those conditions. Fix a sufficiently large universal constant $\iter_0$. We say that a function $r:\mathbb{N}_+ \to \mathbb{N}_+$ is a \emph{well-behaved sub-logarithmic growth rate} if:
\begin{enumerate}
    \item $r$ is sub-logarithmic: $r(\iter) \leq \frac{1}{4} \log (\iter)$ for all $\iter \geq \iter_0$.
    \item $r$ is non-decreasing : $\iter_1 \geq \iter_2 \implies r(\iter_1) \geq r(\iter_2)$ for all $\iter_1,\iter_2 \geq 1$.
    \item $r$ is sub-additive: $r(\iter_1 + \iter_2) \leq r(\iter_1) + r(\iter_2)$ for all $\iter_1,\iter_2 \geq 1$.
\end{enumerate}

Note that all natural sub-logarithmic growth rates such as $\ceil*{\sqrt{\log \iter}}$, $\ceil*{\log \log \iter}$ etc.\ satisfy our definition. We now state the result for well-behaved sub-logarithmic rates.

\begin{theorem} \label{thm:e2e-taxonomy-intro}
    For any well-behaved sub-logarithmic growth rate $r$ there exists a base class $\cF$ such that:
    \begin{enumerate}
        \item $\LD(\cF) = 1$.
        \item $\LD(\cF^{\etoe-\iter}) = \Theta(r(\iter))$ for all $\iter \geq \iter_0$.
    \end{enumerate} 
\end{theorem}

The definition of $\cF$ and the lower bound strategy of the second item is due to unpublished private communication \cite{hanneke2026private}. Theorem~\ref{thm:e2e-taxonomy-intro} handles only sub-logarithmic rates.  For logarithmic rates, we prove the following

\begin{theorem} \label{thm:lit-logarithmic-lower-bound-intro}
    For any $d \in \mathbb{N}_+$, there exists a class $\cF$ such that $\LD(\cF) = O(d)$ and 
    \[
    \VC(\cF^{\etoe-\iter}) = \Omega(\LD(\cF) \log \iter)
    \]
    for all sufficiently large $\iter$.
\end{theorem}

Theorem~\ref{thm:e2e-taxonomy-intro} and Theorem~\ref{thm:lit-logarithmic-lower-bound-intro} are proved in Section~\ref{sec:e2e-taxonomy}.
Since $\VC(\cF) \leq \LD(\cF)$ for all $\cF$, together with the upper bound of Theorem~\ref{thm:e2e-upper-bound}, this establishes that for any $d \in \mathbb{N}_+$ there exists a class $\cF$ such that $\LD(\cF) = O(d)$ and
\[
\VC(\cF^{\etoe-\iter}), \LD(\cF^{\etoe-\iter}) = \Theta(\LD(\cF) \log \iter)
\]
for all sufficiently large $\iter$.

So, Theorem~\ref{thm:e2e-taxonomy-intro}, Theorem~\ref{thm:lit-logarithmic-lower-bound-intro} and Theorem~\ref{thm:e2e-upper-bound-intro} establish an essentially complete taxonomy of the possible growth rates of $\LD(\cF^{\etoe-\iter})$ in terms of $\LD(\cF)$ and $\iter$.

Along the way, Theorem~\ref{thm:lit-logarithmic-lower-bound-intro} resolves an open question on the statistical setting posed by \cite{joshi2025theory}: they proved that for any class $\cF$ it holds that $\VC(\cF^{\etoe-\iter}) = O(\LD(\cF) \log \iter)$, and asked if the logarithmic dependence in their bound can be improved. Theorem~\ref{thm:lit-logarithmic-lower-bound-intro} shows that their bound is tight.

\subsubsection{Corollary for agnostic learning}
Much of the online learning literature studies \emph{agnostic online learning}, where there is no target function in the class that perfectly explains the streamed data. In this setting, instead of measuring the mistake bound, we are usually interested in the \emph{regret bound}, which is the optimal mistake bound measured after subtracting the number of mistakes made by the best function in class. 
A known result of \cite{alon2021adversarial} allow to derive a similar taxonomy for agnostic end-to-end learning.

\begin{corollary} \label{cor:e2e-agnostic-taxonomy-intro}
    Let $\cF$ be a base class. Then for all sufficiently large $\iter$, there exists an agnostic online learner for $\cF^{\etoe-\iter}$ with regret bound
    \[
    O \mleft( \sqrt{T \LD(\cF) \log \iter} \mright).
    \]
    
    Furthermore, for any well-behaved sub-logarithmic growth rate $r$ there exists a base class $\cF$ such that $\LD(\cF) = 1$ and the optimal regret bound for agnostic online learning $\cF^{\etoe-\iter}$ is
    \[
    \Theta \mleft (\sqrt{\rounds \cdot r(\iter)} \mright)
    \]
    for all $\iter \geq \iter_0$ and sufficiently large $\rounds$.

    Finally, for any $d \in \mathbb{N}_+$, there exists a class $\cF$ such that the optimal regret for agnostic online learning of  $\cF^{\etoe-\iter}$ is
    \[
    \Theta \mleft (\sqrt{T d \log \iter} \mright)
    \]
    for all $\iter \geq \iter_0$ and sufficiently large $\rounds$.
\end{corollary}

Corollary~\ref{cor:e2e-agnostic-taxonomy-intro} is proved in Section~\ref{sec:agnostic}.

\subsection{Results for autoregressive linear classifiers}
An interesting specific study case is the basic class of \emph{autoregressive linear classifiers}. This class was presented in \cite{joshi2025theory} and further studied in \cite{hanneke2026sample, joshi2026learning}. It is naturally defined as follows. Let $d \in \mathbb{N}$ be the dimension. For any $\boldsymbol{w} \in \mathbb{R}^d$ and $b \in \mathbb{R}$ define the function $f_{\boldsymbol{w}, b}: \{0,1\}^\star \to \{0,1\}$ by
\[
f_{\boldsymbol{w}, b}(x) = 1 \mleft[ \sum_{i=1}^{\min\{d, |\boldsymbol{x}|\}} \boldsymbol{w}[-i] x[-i] + b \geq 0 \mright]
\]
for all $x\in \{0,1\}^\star$. That is, $f_{\boldsymbol{w},b}$ acts as the linear classifier defined by $(\boldsymbol{w},b)$ when taking only values on the $d$-boolean cube, and the input point to the classifier is given by the (at most) final $d$ bits of the input $x$. The class $\cF_{d, \lin}$ of autoregressive linear classifiers in dimension $d$ is defined as
\[
\cF_{d, \lin} := \{f_{\boldsymbol{w},b}: \boldsymbol{w} \in \mathbb{R}^d, b \in \mathbb{R}\}.
\]

\subsubsection{Mistake and sample complexity bounds}
The results in this section are proved in Section~\ref{sec:linear-mistake}.
\cite{joshi2025theory} proved that $\VC(\cF_{d, \lin}^{\etoe,\iter}) = O(d^2)$ for all $\iter$. This implies that the dependence of the sample complexity of $\ct$-learning  $\cF_{d, \lin}$ on $d$ is also at most $O(d^2)$. \cite{hanneke2026sample} improved the dependence on $d$ for $\ct$-learning from $O(d^2)$ to $O(d)$. No lower bounds (that hold for any $\iter > 1$) were previously given.

In this work, we prove that the optimal mistake bound for online learning $\cF_{d, \lin}^{\etoe}$, both with and without $\ct$ supervision, is $\Theta(d^2)$.

\begin{theorem} \label{thm:linear-class-intro}
    For any dimension $d$ and generation length $\iter$ we have
    \[
    \LD(\cF_{d, \lin}^{\etoe-\iter}) = O(d^2).
    \]
    Furthermore, for any $\iter$ and any $d \geq 3$, there exists an adversary forcing $\Omega(d^2)$ many mistakes on any learner for $\cF_{d, \lin}^{\etoe-\iter}$, even with $\ct$-supervision.
\end{theorem}

The technique we use for the lower bound of Theorem~\ref{thm:linear-class-intro} allows us to derive the first lower bound for the statistical setting that holds for $\iter > 1$.

\begin{theorem}\label{thm:statistical-linear-lower-bound-intro}
    For any dimension $d \geq 3$ and generation length $\iter \geq 2$, any PAC-learner for $\cF_{d, \lin}^{\etoe-\iter}$ has sample complexity $\Omega(d)$, even with $\ct$-supervision.
\end{theorem}

Theorem~\ref{thm:statistical-linear-lower-bound-intro} shows that the upper bound of \cite{hanneke2026sample} for PAC-learning $\cF_{d, \lin}^{\etoe}$ with $\ct$-supervision is tight.

\subsubsection{Computational complexity}
The results in this section are proved in Section~\ref{sec:linear-computational}.
The preceding bounds are information-theoretic. We also obtain a computational
separation between $\etoe$ and $\ct$-learning for linear classifiers. 

\begin{theorem}\label{thm:linear-computational-intro}
Assume Hardness of Learning $\mathsf{TC}^0$.\footnote{This assumption is formally given in Assumption~\ref{ass:hardness-tc0}; for more details, the reader is referred to, e.g.,~\cite{joshi2025theory}.} Then there is no possibly
improper online learner for $\cF_{d,\lin}^{\etoe-\iter}$ that runs in time
polynomial in $d,\iter$ and the input length per round and makes polynomially
many mistakes. In contrast, with $\ct$-supervision, there is a deterministic
polynomial-time learner for $\cF_{d,\lin}^{\etoe-\iter}$ with mistake bound
$O(d^2\log d)$.
\end{theorem}

Thus, $\ct$-supervision removes the computational obstruction, at least up to a
logarithmic factor in the mistake bound.

\subsection{Non-Littlestone classes}

The finite Littlestone assumption is essential. For $\LD(\cF)<\infty$, our upper
bound gives a uniform logarithmic ceiling in $\iter$. Without this assumption,
a single base class can alternate between trivial and unlearnable horizons.

\begin{theorem}\label{thm:alternating-horizons-intro}
There exists a binary base class $\cF$m such that for every $\iter_i=i+4$:
\[
i \text{ odd } \implies \LD(\cF^{\etoe-\iter_i})=0,
\qquad
i \text{ even } \implies \cF^{\etoe-\iter_i} \text{ is unlearnable, even with CoT-supervision.}
\]
That is, the odd horizons are trivial already in the $\etoe$ model, while
the even horizons are unlearnable even with $\ct$-supervision.
\end{theorem}

Thus, outside the Littlestone regime, there are cases where there is no stable growth law in $\iter$.
Theorem~\ref{thm:alternating-horizons-intro} is proved in Section~\ref{sec:non-lit}

\subsection{A stochastic autoregressive lower bound}

We also show that the deterministic picture of online learning does not
extend to stochastic autoregressive generators, at least in the case of  no $\ct$-supervision. In this model, a generator maps
each string in $\Sigma^*$ to a distribution over the next token, so even after the input $x_t$ and
target generator are fixed, the $\ct$ trajectory is random. For the stochastic setting, the natural benchmark is regret. 

\begin{theorem}\label{thm:stochastic-e2e-separation-intro}
For every odd $\iter\geq 1$, there exists a stochastic two-function base class
$\mathcal G=\{g_-,g_+\}$ whose direct next-token problem has regret $O(1)$, but
for every online learner, there exists $g_\star\in\mathcal G$ such that learning
$\mathcal G^{\etoe-\iter}$ against target $g_\star$ incurs expected regret
$\Omega(2^\iter)$ for horizon $\rounds=\Theta(4^\iter)$.
\end{theorem}

Thus, stochastic autoregression can exponentially amplify the difficulty of an
otherwise trivial next-token learning problem. This is in stark contrast to the upper bound for deterministic online learning given in Theorem~\ref{thm:cot-bound-lit-intro} and  Theorem~\ref{thm:e2e-upper-bound-intro}. 
Theorem~\ref{thm:stochastic-e2e-separation-intro} is proved in Section~\ref{sec:stochastic}.

\section{Proof sketches of the main results}
In this section, we explain the main technical ideas of the proofs of our main theorems.
\subsection{Proof sketch of Theorem~\ref{thm:cot-bound-lit-intro}: mistake bound for learning with CoT}
Theorem~\ref{thm:cot-bound-lit-intro} proves that with $\ct$-supervision, learning $\cF^{\etoe-\iter}$ is possible with mistake bound at most $\LD(\cF)$, for all $\iter$. The proof follows from a relatively simple SOA-style\footnote{The \emph{SOA} stands for \emph{standard optimal algorithm}, which is the original algorithm given in \cite{littlestone1988learning} for online learning of binary classes.} argument we describe below.
We design an online learner $\Lrn$ for the class $\cF^{\ct-\iter}$, which operates as follows. $\Lrn$ maintains a \emph{version space} $V_t$ of all generators $f \in \cF$ that are consistent with the feedback given up to the end of round $t$: at first $V_0 = \cF$, and at the end of any round $t$,  $\Lrn$ sets $V_t$ to be all functions in $V_{t-1}$ that are consistent with the feedback of round $t$. In each round $t$, $\Lrn$ receives $x_t$ and predicts
\[
\max _{\hat{y} \in \{0,1\}^\iter}\LD(V_{t-1}|_{x_t \to \hat{y}}),
\]
where $V_{t-1}|_{x_t \to \hat{y}}$ denotes the restriction of $V_{t-1}$ only to functions for which autoregressive chain-of-thought on $x_t$ is precisely $\hat{y}$. The crucial observation is that there is at most one $\hat{y} \in \{0,1\}^\iter$ such that
\[
\LD(V_{t-1}|_{x_t \to \hat{y}}) = \LD(V_{t-1}).
\]
Indeed, if there are two distinct $\hat{y}, \hat{y'}$ such that
\[
\LD(V_{t-1}|_{x_t \to \hat{y}}) = \LD(V_{t-1}|_{x_t \to \hat{y'}}) = \LD(V_{t-1})
\]
then one can construct a Littlestone tree of depth $\LD(V_{t-1}) + 1$ which is shattered by $V_{t-1}$, which is a contradiction to the fact that $\LD(V_{t-1})$ is the optimal mistake bound for learning $V_{t-1}$, as proved by \cite{daniely2015multiclass}.

Therefore, after every mistake of $\Lrn$ the Littlestone dimension of the version space decreases by at least $1$, which implies an upper bound of $\LD(V_0) = \LD(\cF)$.

\subsection{Proof sketch of Theorem~\ref{thm:e2e-upper-bound-intro}: mistake bound for learning without CoT}
Consider a Littlestone tree $\tree$ of depth $m$ that is shattered by $\cF^{\etoe-\iter}$. It suffices to show that $m = O( \LD(\cF) \log (\LD(\cF) \iter))$. The idea is a variant of tree-covering arguments occasionally used in online learning (see \cite{filmus2023optimal} for example). First, we inflate $\tree$ to represent the entire chain-of-thought for every instance labeling a node in the tree. We do so by replacing each node and its two outgoing edges by the the full $\iter$-step \emph{generation tree} matching the instance labeling the node. An $\iter$-step generation tree for $x \in \{0,1\}^\star$, denoted by $\tree_{\iter}(x)$, is a Littlestone tree of depth $\iter$ with the root labeled by $x$, and for any prefix $u \in \{0,1\}^{<\iter}$, the node at the end of the prefix $u$ is labeled by $x \circ u$. So the branches of $\tree_{\iter}(x)$ represent the different possible chain-of-thoughts for $x$. Let $\tree_{\cF}$ be the inflated tree, which is now of depth $m \iter$. We now analyze the size of the set of branches in $\tree_{\cF}$ that are realized by some $f \in \cF$, denoted as $B_{\cF}(\tree_{\cF})$. Since every generation tree in $\tree_{\cF}$ is originated from a node in $\tree$ and its two outgoing edges, and since $\tree$ is shatterd by $\cF^{\etoe-\iter}$, it is not hard to see that $B_{\cF}(\tree_{\cF})$ grows by a factor of at least $2$ after each $\iter$ levels of $\tree_\cF$, which implies
\[
|B_{\cF}(\tree_{\cF})| \geq 2^m.
\]
On the other hand, a known (for example, by \cite{bhaskar2021thicket}) Sauer-Shelah-Perles lemma for trees give
\[
|B_{\cF}(\tree_{\cF})| \leq \binom{m \iter}{ \leq \LD(\cF)}.
\]
By the two inequalities, we obtain
\[
2^m \leq \binom{m \iter}{ \leq \LD(\cF)}.
\]
Solving the inequality for $m$ gives the stated bound.
\subsection{Proof sketch of Theorem~\ref{thm:e2e-taxonomy-intro}: taxonomy for sub-logarithmic rates}
To prove Theorem~\ref{thm:e2e-taxonomy-intro}, we construct a class $\cF$ (suggested in private communication \cite{hanneke2026private})
that has two desired properties:
\begin{enumerate}
    \item It is simple, and in fact has $\LD(\cF) = 1$.
    \item We can make $\cF^{\etoe-\iter}$ as complicated as we want up to some well-behaved sub-logarithmic rate $r(\iter)$. In fact, given any well-behaved sub-logarithmic rate $r$, we have  $\LD(\cF^{\etoe-\iter}) = \Theta(r(\iter))$.
\end{enumerate}

To implement this, we define a fixed baseline function $f$, and $\cF$ is constructed from a class of functions such that each function differs from $f$ on precisely one instance. This fact alone suffices to prove the first item. To prove the second item, the baseline function $f$, as well as the other functions in the class, should be carefully chosen to fit the given well-behaved sub-logarithmic rate $r$. The proof, and even the definition of $\cF$, are quite technical. Below, we explain the high level idea of the definition and proof. The function $f$ is defined such that it is ``interesting" only on instances of the form
\[
 x_{s,k,z} := 0^s 1 0 ^k 1 0^{s-k-1}z
\]
where $s$ can be any natural number, $k \in \mleft[ 2^{r(s)} \mright]$, and $z \in \{0,1\}^{< r(s)}$. Leaving some technical details aside, $f(x_{s,k,z}) = k_{|z|}$, so $f(x_{s,k,z})$ in fact returns a bit of the binary representation of $k$, in a location encoded by the length of $z$. For any other instance $x$, define $f(x) = 0$.  Note that assuming that $r$ is sub-logarithmic is crucial already here: since $k$ can be large as $2^{r(s)}$, assuming that $r$ is sub-logarithmic ensures that $s-k-1 \geq 1$, and so $0^{s-k-1}$, which appears in $s_{s,k,z}$, is a well-defined binary string. Now, for every $s,k$ as above we define a function $f_{s,k}$. This function is defined exactly as $f$ everywhere except for the instance $0^s 1 0^k$, for which $f$ returns $0$, but we define $f_{s,k}(0^s 1 0^k) = 1$. The class $\cF$ is defined as the class of all $f_{s,k}$. The definition of $f_{s,k}$ is one of the core ideas of the proof: when the instance $0^s 1 0^k$ is given to $f_{s,k}$ and only to $f_{s,k}$, the $s$-step autoregressive generation of  $f_{s,k}$ for $0^s 1 0^k$ is non-degenerate, and in fact encodes the binary representation of $k \in \mleft[ 2^{r(s)} \mright]$. This is the high-level idea of the lower bound: for $s := \iter$, any learner makes at least roughly $\log 2^{r(\iter)} = r(\iter)$ many mistakes before identifying the ``correct" $k$, where the logarithmic lower bound comes from optimality of the halving algorithm in this case.

For the upper bound, the main idea is to design a learner for which the halving algorithm used to find the correct $k$ is essentially the only obstacle. Let us describe this learner in high level. Until making its first mistake, the learner always predicts as $f^{\etoe-\iter}$. Once the learner makes its first mistake on some instance $x$, it knows the following fact:
\[
f_{s^\star, k^\star}^{\etoe-\iter}(x) \neq f^{\etoe-\iter}(x),
\]
where $f_{s^\star, k^\star}$ is the target function. Since $f_{s^\star, k^\star}$ differs from $f$ only on the instance $0^{s^\star} 1 0^{k^\star}$, this implies that $x$ must be of the form $0^{s^\star} 1 0^i$ for some $i \in \mathbb{N}$. Therefore, the learner now knows $s^\star$ just by looking on $x$, and so the learner now only needs to find $k^\star$ to know the target function $f_{s^\star, k^\star}$. If $s^\star = O(\iter)$, then $k^\star$ can only receive as few as
\begin{equation}\label{eq:sketch-upper-bound}
    2^{r(s^\star)} \leq 2^{r(O(\iter))} \leq 2^{O(r(\iter))}
\end{equation}
many values, and thus the halving algorithm finds $k^\star$ after at most $O(r(\iter))$ many mistakes, which is analogous to the argument of the lower bound. Note that \eqref{eq:sketch-upper-bound} does not hold for a general rate $r$. The first inequality holds since $r$ is non-decreasing, and the second inequality holds since $r$ is sub-additive. This is the exact reason that we make those two assumptions.  However, it is also possible that $s^\star \gg \iter$. In this case, we need to use a completely different argument. In the case that $s^\star \gg \iter$, the learner acts as follows. For any $x$ not of the form $0^{s^\star} 1 0 ^i$ it predicts as $f$ and never errs. For $x$ which is of the form $0^{s^\star} 1 0 ^i$, the learner always predicts $0$ until making its first mistake on such an instance, which implies that
\begin{equation} \label{eq:sketch-interesting-part}
   f_{s^\star, k^\star}^{\etoe-\iter}(0^{s^\star} 1 0 ^i) = 1,
\end{equation}

where $0^{s^\star} 1 0 ^i$ is the instance the learner errs on. However, recalling the definition of $f$, it is non-degenerate only for instances of the form $0^s 1 0 ^k 1 0^{s-k-1}z$, which in particular means that in our case $s^\star - k - 1$ zeroes must be appended before getting to the interesting part. However, since $s^\star \gg \iter$, we have that $s^\star - k - 1 > \iter$, and thus  $\iter$ is simply not large enough to even get to the interesting part. So, in fact, the case \eqref{eq:sketch-interesting-part} can essentially never happen (leaving aside some technical details).

\subsection{Proof sketch of Theorem~\ref{thm:lit-logarithmic-lower-bound-intro}: logarithmic lower bound}
Theorem~\ref{thm:lit-logarithmic-lower-bound-intro} states that for any (large enough) $d$ there exists a base class $\cF$ with $\LD(\cF) = O(d)$, and such that $\VC(\cF^{\etoe-\iter}) = \Omega(d \log \iter)$ for all (large enough) $\iter$. The idea is to construct for every $d,\iter$ a base class $\cF:=\cF(d,\iter)$ of size $\approx \iter^d$ which  is very unbalanced for every instance $x$ (that is, the functions of $\cF$ demonstrate an overwhelming majority towards one of the possible labels for every $x$), but after $\iter$ generation steps they become much more balanced for many instances. This is precisely what makes $\LD(\cF) = O(d)$, while $\VC(\cF^{\etoe-\iter}) = \Omega(d \log \iter)$.
In more detail, our specific choice of overwhelming majority is of a $\approx(1-1/\iter)$-fraction of functions for all $x$. This indeed gives $\LD(\cF) = O(d)$, since we start from $\approx \iter^d$ many functions, and for each restriction of the form $\cF_{x \to r} := \{f \in \cF: f(x) = r\}$ for some $r \in \{0,1\}$, the number of surviving functions decreases by a factor of $\iter$. Therefore, after at most roughly $d$ such restrictions (that correspond to an optimal learner's mistakes), there are no surviving functions left. On the other hand, we ensure that after $\iter$ many generation steps, there is a constant fraction of functions realizing each label for $x$, for many instances $x$. Since for every restriction of an instance to be classified with some label after $\iter$ generation steps a constant fraction of the class stays consistent, it is implied that
\[
\VC(\cF^{\etoe-\iter}) = \Omega(\log (\iter^d)) = \Omega(d \log \iter).
\]

It remains to construct a class having the above desired overwhelming majority property for all $x$.
However, ensuring this desired property simultaneously for all $x$ seems difficult to do deterministically. We thus define the class at random: every $f$ gives to every $x$ (in a specific desired set of instances) $r\in \{0,1\}$ with probability $1-1/\iter$, and $1-r$ with probability $1/\iter$. Figure~\ref{fig:hard-class} demonstrates why $\cF_{x \to 0}^{\etoe-\iter}$ and $\cF_{x \to 1}^{\etoe-\iter}$ both contain a constant fraction of $\cF$.

\begin{figure}
    \centering

\begin{tikzpicture}[
  grow=down,
  level distance=17mm,
  edge from parent/.style={draw, line width=1.8pt},
  internal/.style={coordinate},
  vtx/.style={circle, draw=none, minimum size=9mm, inner sep=0pt, font=\ttfamily\scriptsize},
  rednode/.style  ={vtx, fill=red!75},
  greennode/.style={vtx, fill=green!60!black},
  toRed/.style={edge from parent/.append style={draw=blue!85!black}},
  level 1/.style={sibling distance=60mm},
  level 2/.style={sibling distance=40mm},
  level 3/.style={sibling distance=28mm},
  level 4/.style={sibling distance=20mm},
]

\node[internal, label=above:{\ttfamily 0}] (root) {}
  child { node[internal] (n00) {}
    child { node[internal] (n000) {}
      child { node[internal] (n0000) {}
        child[toRed] { node[vtx, fill=black!15] (n00000) {} }
        child        { node[greennode] (n00001) {$\mathtt{0^{4}1}$} }
      }
      child[toRed] { node[rednode] (n0001) {$\mathtt{0^{3}1}$} }
    }
    child[toRed] { node[rednode] (n001) {$\mathtt{0^{2}1}$} }
  }
  child[toRed] { node[rednode] (n01) {$\mathtt{01}$} };

\coordinate (base) at (n00000.center);

\foreach \v in {n01,n001,n0001} {
  \draw[dashed, line width=1.1pt]
    ($(\v.south)+(0,-1mm)$) -- ($(\v.south |- base)+(-6mm,0)$);
}

\end{tikzpicture}

\caption{The generation tree $\tree_{\iter}(x^{(j)})$, visualized for the hard class $\cF_{d,\iter}$ with $\iter=4$ and $j=1$ (that is, $x^{(j)} = 0$). Note that there are  $\iter-1$ red nodes, where the version space of each typically contains an $O(1)/\iter$ fraction of the class, and thus in total a $(\iter-1) \frac{O(1)}{\iter} = O(1)$ fraction of the class is contained in the union of red vertices. The edges where the size of the class drops by a factor of $O(1)/\iter$ are blue. In the rest of the edges, the size of the version space only drops by a factor of at least $1-O(1)/\iter$. Therefore, a fraction of $\mleft(1-O(1)/\iter \mright)^\iter = O(1)$  of the class is typically contained in the green vertex version space. Furthermore, for all $f$ in the version space of the red nodes, we have $f^{\etoe-\iter}(0) = 0$, and  for all $f$ in the version space of the green node, we have $f^{\etoe-\iter}(0) = 1$. So, for each possible choice of labeling for $0$, we get to keep a constant factor of the version space. Applying this idea repeatedly allows to shatter a set of size $\Omega(\log (\iter^d)) = \Omega(d \log \iter)$.}
\label{fig:hard-class}
\end{figure}

In this intuitive explanation, one of the important facts we overlooked is that $\cF:= \cF(d,\iter)$, that is, it depends on the generation length $\iter$, while we want the result to hold for all $\iter$. This is resolved by ``gluing" together $\{\cF(d,\iter)\}_{\iter}$ for all $\iter$ values, a technique used for example in \cite{chase2024deterministic}.

\subsection{Proof sketches of Theorem~\ref{thm:linear-class-intro} and Theorem~\ref{thm:statistical-linear-lower-bound-intro}: bounds for linear classifiers}
The interesting part of our bounds for autoregressive linear classifiers is the following ``latch" mechanism allowing to ``drag" the output of the base classifier throughout the entire autoregressive process in the price of reducing the dimension by $2$:

\begin{lemma}[Latch lemma] \label{lem:latch-sketch}
    Fix $d \geq 3$ and denote $m:= d-2$. Then, for all $\boldsymbol{v} \in \mathbb{R}^m$ and $c \in \mathbb{R}$, there exists $f_{\boldsymbol{w},b} \in \cF_{d, \lin}$  such that
    \[
    f_{\boldsymbol{w},b}^{\etoe-\iter}(z 1 0) = f_{\boldsymbol{v},c}(z)
    \]
    for all $z \in \{0,1\}^m$ and all $\iter \geq 2$.
\end{lemma}

Once we have Lemma~\ref{lem:latch-sketch}, the rest of the arguments in the proofs of Theorem~\ref{thm:linear-class-intro} and Theorem~\ref{thm:statistical-linear-lower-bound-intro} are either simple or known. Indeed, the upper bounds for autoregressive linear classifiers are not larger than the standard upper bounds for (standard) linear classifiers. Therefore, it essentially suffices to show that the autoregressive process cannot make the class simpler, and the class maintains its original mistake bound/sample complexity.

The idea of the proof of Lemma~\ref{lem:latch-sketch} is quite simple. We define $\boldsymbol{w}$ to be as $\boldsymbol{v}$, except for the final two locations of $\boldsymbol{w}$, that we are free to choose. We also choose $b$ to be as $c$, but shifted in some value. When choosing the above values carefully, we are able to ensure that:
\begin{enumerate}
    \item  For all $z \in \{0,1\}^m$: $f_{\boldsymbol{w},b}(z 10) = f_{\boldsymbol{v},c}(z)$ (first  bit of $\ct$ is correct).
    \item  For any $x \in \{0,1\}^\star$ that ends with $00$, we have $f_{\boldsymbol{w},b}(x) = 0$ (always drag $0$).
    \item For any $x \in \{0,1\}^\star$ that ends with $1$, we have $f_{\boldsymbol{w},b}(x) = 1$ (always drag $1$).
\end{enumerate}

The above three items essentially establish the statement of the lemma. The idea is to use known lower bounds for dimension $m$, but add the suffix $10$ to every instance used in the proof, so the dimension becomes $d := m+2$. Now, the first item ensures that for every $z\in \{0,1\}^m$, $f_{\boldsymbol{w},b}(z 10)$ outputs the correct intended output $f_{\boldsymbol{v},c}(z)$. The other two items then establish that this correct output is being dragged throughout the entire autoregressive generation process: if $f_{\boldsymbol{w},b}(z 10)=0$ then the generated new instance must end at $00$, and if $f_{\boldsymbol{w},b}(z 10)=1$ it must end with $1$. Thus, any lower bound strategy for $\cF_{m, \lin}$ works also for $\cF_{d, \lin}^{\etoe-\iter}$ for all $\iter$, even with $\ct$-supervision.
\section{Future work}

Our work leaves some natural and interesting directions for future work. Below, we describe some of them.

\paragraph{Extend the results to the multiclass setting.}
Some of our results extend to the multiclass setting $|\Sigma| > 2$ ``for free", such as the upper bound for $\ct$-learning (Theorem~\ref{thm:cot-bound-lit-intro}). This is not true, however, for all results in this paper. For example, our initial attempts to extend the upper bound for $\etoe$-learning (Theorem~\ref{thm:pac-cot-bound-lit-intro}) to the multiclass setting demonstrated an extra $\log |\Sigma|$ factor, which we do not know if necessary, even for some classes.

Perhaps an even more interesting direction in the multiclass setting is to study the optimal mistake bounds with and without $\ct$-supervision attainable with different models of partial feedback, such as bandit feedback and feedback graphs.


\paragraph{The stochastic setting.} Theorem~\ref{thm:stochastic-e2e-separation-intro} shows that in the stochastic setting, the regret bound can grow exponentially with $\iter$ even for a class of size two. Further characterizing the online and PAC learnability of the stochastic setting is a natural direction for a follow-up work.  

\paragraph{Studying specific classes.}
A natural direction is to move beyond worst-case base classes and analyze
standard finite-Littlestone classes directly, as we did by studying autoregressive linear threshold functions. Other concrete examples for future work include
Boolean classes such as conjunctions, disjunctions, decision lists, and
bounded-depth decision trees, as well as finite-domain geometric classes such
as intervals and unions of intervals. For such structured classes, one may
hope for sharper mistake bounds than the general
$O(\LD(\cF)\log(\LD(\cF)\iter))$ guarantee, or for class-specific separations
between $\etoe$-learning and $\ct$-learning.

\section{Additional related work}

Beyond the closest autoregressive PAC works discussed above, we mention three surrounding directions studied extensively in recent years.

\paragraph{Chain-of-thought data in language-model training.}
Chain-of-thought and scratchpad-style supervision were introduced empirically as ways to expose intermediate computations in language models~\cite{Nye+21,wei2022chain}. This is related to supervised fine-tuning and next-token prediction more broadly~\cite{Bro+20,Wei+22}, following the standard transformer language-modeling paradigm~\cite{Vas+17,Rad+19}. Recent RL-based post-training methods also encourage longer reasoning traces~\cite{Jae+24,Deep+25}, but their supervision signal is typically attached to final outcomes rather than to every intermediate token. The cost and importance of curating chain-of-thought data is also emphasized in large-scale model reports~\cite{Dub+24}. Our results formalize this distinction in an online model: observing only the final generated token can increase the mistake bound with the generation length $\iter$, while full $\ct$ feedback removes this dependence for finite-Littlestone base classes.

\paragraph{Theory of chain-of-thought.}
A growing theoretical literature studies why intermediate reasoning traces can help. This includes universality of autoregressive next-token predictors~\cite{malach2023auto}, sample-complexity bounds in terms of CoT information~\cite{altabaa2025cot}, and reinforcement learning after next-token prediction~\cite{tsilivis2025reinforcement}. Other works study the power and limitations of transformer architectures with chain-of-thought reasoning~\cite{merrill2023expressive,li2024chain,chen2025theoretical,kim2024transformers,huang2025transformers}. These works are mostly statistical, architectural, or optimization-theoretic. In contrast, our results are architecture-free and online: they characterize the possible mistake-bound growth rates of autoregressive learning under end-to-end feedback and show when trajectory feedback eliminates this growth.

\paragraph{Computational hardness.}
Our computational separation is related in spirit to the broader literature on hardness of distribution-free learning; see~\cite{servedio2025probably} for a recent survey. Relevant examples include hardness results for intersections of halfspaces~\cite{klivans2009cryptographic,tiegel2024improved}, and DNFs or related classes under average-case or pseudorandomness assumptions~\cite{applebaum2010public,daniely2016complexity,daniely2021local,applebaum2025structured}. Unlike this hardness literature, most of our paper is information-theoretic; the computational result shows that, even for autoregressive linear classifiers, end-to-end feedback can leave a computational obstruction that $\ct$ feedback removes.

\section{Technical preliminaries} \label{sec:prel}
The online autoregressive model studied in this work was formally defined in Section~\ref{sec:model}. The goal of this section is to formally define additional relevant concepts from learning theory, as well as to state known and relevant classic results. Most of the definitions and notation are borrowed from \cite{hanneke2026sample}.

\subsection{Littlestone trees}

In this paper, a Littlestone tree (or just a tree) $\tree$ is a finite, full, and rooted binary tree, accompanied with the following information:
\begin{enumerate}
    \item Every internal node $v$ is associated with an instance $x \in \cX$, denoted as $x_v$.
    \item For every internal node $v$, the left outgoing edge is labeled with $0$, and the right outgoing edge with $1$.
\end{enumerate}
The set of vertices and edges of $\tree$ are denoted by $v(\tree)$ and $e(\tree)$, respectively.
Every path in the tree that starts at the root is called a prefix. A prefix that ends at a leaf is called a \emph{branch}. The set of all branches in $\tree$ is denoted $\cB(\tree)$. A prefix can be identified by a sequence of vertices starting from the root, or by just the last vertex in the prefix, or by the associated sequence of edges.
A prefix $p = v_0, \ldots, v_t$ in $\tree$ defines a sample $(x_1,y_1), \ldots, (x_t,y_t) \in \cX \times \cY$ in the natural way: $x_i = x(v_{i-1})$, and $ y_i$ is the label of the edge $(v_{i-1},v_i)$ for all $i$. We denote $\boldsymbol{x}(p) = x_1, \ldots, x_t$ and $\boldsymbol{y}(p) = y_1,\ldots y_t$. For a given vertex $v$, we use $p_\tree(v)$ (or $p(v)$, when the tree $\tree$ is fixed) to denote the path from the root to $v$. We may overload notation and use $p(v)$ to denote the sample associated with the path, when the context is clear. If there is a bijection between the set of vertices and the set of their instance labelings, we may replace $v$ with its instance label in all notation. 

\paragraph{Generation trees.}
Given a base class $\cF$, a generation length $\iter$, and an instance $x$, there is a natural way to represent all possible labels that could be given to $x$ by functions in $\cF^{\mathsf{CoT}-\iter}$ as a decision tree $\tree(x) := \tree_{ \iter}(x)$. We call this tree the \emph{generation tree} of $x$ of depth $\iter$, or simply the \emph{generation tree} of $x$, if $T$ is fixed.  The tree is constructed as follows. The root of $\tree(x)$ is $x$. The left child of $x$ is labeled by $x0$ and likewise, the right child of $x$ is labeled by $x1$. We continue this labeling process inductively on the children of $x$, for a total number of $T$ many times. There is a bijection between the set $\{\boldsymbol{y}(b)\}$ for all branches $b$ in $\tree(x)$ and the label space $\{0,1\}^T$. Furthermore, note that a function $f^{\mathsf{CoT}-T} \in \cF^{\mathsf{CoT}-T}$ realizes a branch $b$ if and only if $f^{\mathsf{CoT}-T}(x) = y(b)$. The branch $b$ may be referred to as the \emph{computation path} of $f^{\mathsf{CoT}-T}$ on $x$. This representation will be useful in some of the analyses conducted in this work. See a visualized example of a generation tree in Figure~\ref{fig:arl}.

\begin{figure}
    \centering
\begin{tikzpicture}[
  level distance=25mm,
  sibling distance=55mm,
  level 2/.style={sibling distance=35mm},
  every node/.style={font=\Large, inner sep=1pt},
  edge from parent/.style={draw, line width=0.8pt},
]
\node {0}
  child { node {00}
    child { node {000} }
    child { node {001} }
  }
  child { node {01}
    child { node {010} }
    child { node {011} }
  };
\end{tikzpicture}
    \caption{The tree $\tree_{\iter}(x)$ for $\iter=2, x=0$.}
    \label{fig:arl}
\end{figure}

\subsection{Class dimensions, learnability, and an SSP lemma} \label{sec:prel-dimensions}

Consider a tree $\tree$ and a class $\cF \subset  \{0,1\}^\cX$.
Given a function $f\in \cF$, we say that $f$ \emph{realizes} a prefix $p= v_0, \ldots, v_t$ in $\tree$ if $f(\boldsymbol{x}(p)_i) = \boldsymbol{y}(p)_i$ for all $i$. The set of all functions $f \in \cF$ who realizes $p$ are called the \emph{version space} of $v_t$ (or of $p$, or of $(\boldsymbol{x}(p), \boldsymbol{y}(p))$) with respect to $\cF$, which is denoted as $V_{\cF}(v_t)$, or simply $V(v_t)$ when $\cF$ is fixed. We say that $\cF$ \emph{shatters} $\tree$, if for every leaf $\ell \in \tree$ it holds that $V(\ell) \neq \emptyset$. The set of trees shattered by $\cF$ is denoted by $\cT(\cF)$. The set of branches in $\tree$ that are realized by $\cF$ are denoted by $B_{\cF}(\tree)$. That is, if $\cF$ shatters $\tree$ then $B_{\cF}(\tree) = \tree$. We now define two fundamental and useful dimensions.

\begin{definition}[Littlestone dimension]
The \emph{Littlestone dimension} of $\cF$, denoted $\LD(\cF),$ is the maximal depth of a perfect Littlestone tree which is shattered by $\cF$. If there is no such maximal depth, we say that $\LD(\cF) = \infty$.
\end{definition}

The Littlestone dimension is known to characterize \emph{online learnability}.
\begin{definition}[Online learnability]
    In the online learning model, in each round $t$ an adversary sends $x_t \in \cX$ to the learner. The learner then predicts\footnote{We only discuss deterministic learners.} $\hat{y}_t \in \cY$, and the adversary then sends the correct label $y_t \in \cY$. Finally, the learner suffers the 0/1 loss $1[\hat{y}_t \neq y_t]$. The adversary is restricted to be \emph{realizable}: there must be $f \in \cF$ so that all examples $(x_t,y_t)$ satisfy $y_t = f(x_t)$. A learner's \emph{mistake bound} is the maximal number of mistakes it makes in this setting, against any adversary. We say that $\cF$ is \emph{online learnable} with mistake bound $L$ if there exists a learner with mistake bound $L$ such that $L < \infty$.
\end{definition}

 We will use the following known characterization.

\begin{theorem}[Online learnability characterization \cite{littlestone1988learning}] \label{thm:online-characterization}
    The optimal mistake bound for online learning $\cF$ is precisely $\LD(\cF)$. That is, there exists a learner with mistake bound $\LD(\cF)$, and there exists an adversary who can force at least $\LD(\cF)$ many mistakes on any learner.
\end{theorem}

For Littlestone classes, there is a useful SSP (Sauer-Shelah-Perles) lemma for trees, proved e.g.\ by \cite{bhaskar2021thicket}.

\begin{theorem}[SSP for trees \cite{bhaskar2021thicket}] \label{thm:trees-ssp}
    Let $\cF \subset  \{0,1\}^\cX$ be a class, and let $\tree$ be a Littlestone tree of depth $n$ with vertices labeled by instances from $\cX$. Then
    \[
    |B_{\cF}(\tree)| \leq \binom{n}{\leq \LD(\cF)} \leq n^{2 \LD(\cF)}.
    \]
\end{theorem}

The work of \cite{daniely2015multiclass} further extended the definition of the Littlestone dimension to non-binary classes $\cF \subset  \cY^\cX$, $|\cY| \geq 2$. In their definition, instead of labeling each left edge with $0$ and each right edge with $1$, one may choose for every two sibling edges (two edges that are the outgoing edges of a single vertex) two distinct labels from $\cY$ to be labeled by. When $\cY = \{0,1\}$, this definition is reduced to the standard definition. However, for $|\cY| > 2$, \cite{daniely2015multiclass} showed that this definition still characterizes the optimal mistake bound for online learning $\cF$.

\begin{theorem}[Multiclass Online learnability characterization \cite{daniely2015multiclass}] \label{thm:online-multiclass-characterization}
    The optimal mistake bound for online learning $\cF \subset  \cY^\cX$ for any $\cY$ is precisely $\LD(\cF)$.
\end{theorem}

Some of our results have implications for autoregressive PAC-learning, so for completeness we also formally define PAC-learnability here.

\begin{definition}[PAC learnability] \label{def:pac-learnability}
We say that $\cF \subset \cY^{\cX}$ (where $\cY$ is not necessarily binary) is \emph{PAC-learnable} with sample complexity $m:= m(\epsilon,\delta)$, if there exists a PAC-learner $\Lrn: (\cX \times \cY)^\star \to \cY^{\cX}$, such that for every distribution $D$ over $\cX$, for every $\epsilon, \delta >0$, and for every target function $f_\star \in \cF$, if $\boldsymbol{x}:=x_1, \ldots, x_m \sim_{iid} D$ and $(\boldsymbol{x}, \boldsymbol{y}):= ((x_i, f_\star(x_i))_{i=1}^m$, then
\[
L_{D, f_\star}(f_{\Lrn(\boldsymbol{x}, \boldsymbol{y})}) > \eps
\]
with probability at most $\delta$, where $f_{\Lrn(\boldsymbol{x}, \boldsymbol{y})} := \Lrn(\boldsymbol{x}, \boldsymbol{y})$ and $L_{D, f_\star}(f) := \Pr_{x \sim D} \mleft[ f(x) \neq  f_\star(x) \mright]$.
\end{definition}

Our results have implications for $\etoe$-PAC learning, which is simply PAC-learning of $\cF^{\etoe-\iter}$, and also for $\ct$-PAC learning, which is the same as $\etoe$-PAC learning, with the difference that training examples reveal the entire chain-of-thought led to the final output. So, as in the online setting, every PAC sample complexity bound for learning $\cF^{\ct-\iter}$ is also A PAC sample complexity bound for $\ct$-learning, which is an easier task than the task of PAC-learning $\cF^{\ct-\iter}$.

We now define the VC-dimension, which is known to characterize the sample complexity of PAC-learning binary classes ($|\cY| = 2$), up to accuracy ($\epsilon, \delta$) factors  \cite{blumer1989learnability, ehrenfeucht1989general, hanneke2016optimal}.

The VC-dimension is defined as the Littlestone dimension, but restricted only to trees where each branch is labeled by the same sequence of instances. A formal definition is given below.

\begin{definition}[VC-dimension]
The \emph{VC dimension} of $\cF$, denoted $\VC(\cF)$, is the maximal depth of a perfect tree which is shattered by $\cF$, and also satisfies $\boldsymbol{x}(b_1) = \boldsymbol{x}(b_2)$ for every two branches $b_1,b_2$. If there is no such maximal depth, we say that $\VC(\cF) = \infty$.
\end{definition}
Since in the context of the VC-dimension we require that $\boldsymbol{x}(b_1) = \boldsymbol{x}(b_2)$ for every two branches $b_1,b_2$, instead of discussing trees, we may simply refer to $\boldsymbol{x}(b_1)$ as a \emph{shattered set}. Note that the definition of the VC-dimension immediately implies that $\VC(\cF) \leq \LD(\cF)$ for any class $\cF$.

As the result below states, analyzing the Littlestone dimension of a class $\cF$ could have an advantage in the case where $|\cY| > 2$.

\begin{theorem}[Online-to-PAC conversion \cite{littlestone1989from}] \label{thm:online-to-pac}
    If $\LD(\cF) < \infty$, then $\cF$ is learnable with sample complexity
    \[
    m(\eps, \delta) = O \mleft( \frac{\LD(\cF) + \log (1/\delta)}{\eps} \mright).
    \]
\end{theorem}

In the case $|\cY|=2$, Theorem~\ref{thm:online-to-pac} does not give any advantage, since \cite{hanneke2016optimal} proved the same bound with $\LD(\cF)$ replaced by $\VC(\cF)$, and $\VC(\cF) \leq \LD(\cF)$. However, Theorem~\ref{thm:online-to-pac} applies also to the case $|\cY| > 2$, while bounds depending on the VC-dimension are not, since the VC-dimension is defined only over binary classes.

\section{Upper bound for end-to-end learning} \label{sec:e2e-upper-bound}

\begin{theorem} \label{thm:e2e-upper-bound}
    Let $\cF$ be a base class. Then for all $\iter \geq 2$:
    \[
    \LD(\cF^{\etoe - \iter}) = O(\LD(\cF) \log (\iter \LD(\cF))).
    \]
\end{theorem}

\begin{proof}
    Let $\tree$ be a Littlestone tree of depth $m$ which is shattered by $\cF^{\etoe - \iter}$. We will show that $m = O(\LD(\cF) \log (\iter \LD(\cF)))$, which implies the lemma. For every node $u \in \{0,1\}^{ < m}$ of $\tree$ (where the node is given by the labels of the edge-path from the root to the node), let $x_u$ be the instance labeling that node. Now, we inflate the tree $\tree$ to a tree $\tree_{\cF}$ of depth $m \iter$ that takes into account the intermediate autoregressive steps. The idea is to inflate every two sibling edges in $\tree$ to possible paths of length $\iter$ using a generation tree. Let us formally describe the recursive construction of $\tree_\cF$. We maintain a saved path $v \in \{0,1\}^{<m}$ in the original tree $\tree$ saved in every node  of $\tree_\cF$. In the root of $\tree_\cF$, the saved path is $\emptyset$ (the empty path), and every node inherits the saved path from its parent, until we explicitly update the saved path. For the base of the recursion, we construct the generation tree $\tree_{\iter}(x_\emptyset)$. For the recursive step, we either finish construction, or paste another generation tree at every leaf of the tree constructed so far. Concretely, for every leaf of the tree constructed so far, if $v$ is the saved path of this leaf and $|v| = m$, finish construction. Otherwise, if it is a left leaf (that is, a leaf that lies at the tip of an edge labeled with $0$), paste $\tree_{\iter}(x_{v0})$ at this leaf, and update the saved path to $v0$. Otherwise, paste $\tree_{\iter}(x_{v1})$ at this leaf, and update the saved path to $v1$. 

    For a class $\cF$ and a tree $\tree$, let $B_{\cF}(\tree)$ be the set of branches of $\tree$ that are realized by $\cF$. Now, note that
    \begin{equation} \label{eq:branches-lower-bound}
        |B_{\cF}(\tree_\cF)| \geq 2^m.
    \end{equation}
    Indeed, there is a one-to-one function $G$ from the branches of $\tree$ to the branches of $B_{\cF}(\tree_\cF)$. For a branch $u \in \{0,1\}^m$ of $\tree$, let $f_u \in \cF$ such that $f_u^{\etoe-\iter}$ realizes $u$. Then by construction of $\tree_\cF$, $f_u \in \cF$ agrees with some branch in $\tree_\cF$ that its leaf has the saved path $u$. So, $G(u)$ returns this branch. Therefore $G(u) \in B_{\cF}(\tree_\cF)$ and $G$ is clearly a bijection. 
    On the other hand, Theorem~\ref{thm:trees-ssp} implies
    \begin{equation} \label{eq:branches-upper-bound}
        |B_{\cF}(\tree_\cF)| \leq  (m \iter)^{2 \LD(\cF)}.
    \end{equation}
    From \eqref{eq:branches-lower-bound} and \eqref{eq:branches-upper-bound}, we obtain:
    \[
    2^m \leq (m \iter)^{2 \LD(\cF)}.
    \]
    Solving this inequality for $m$ implies the stated bound.
\end{proof}

\section{The taxonomy of end-to-end learning } \label{sec:e2e-taxonomy}

We have shown that for any base class $\cF$ and any generation length $\iter \geq \LD(\cF)$ we have $\LD(\cF^{\etoe-\iter}) = O(\LD(\cF) \log \iter)$. Below, we prove that essentially any ``well-behaved" sub-logarithmic growth rate of $\LD(\cF^{\etoe-\iter})$ between constant and logarithmic in $\iter$ is realized by some Littlestone class. We further show that for all $d$, there exists a class $\cF$ with  $\LD(\cF) = O(d)$ and 
\begin{equation} \label{eq:lit-vc-bound}
    \LD(\cF^{\etoe-\iter}) \geq \VC(\cF^{\etoe-\iter}) = \Omega(\LD(\cF) \log \iter)
\end{equation}
for all $\iter$.
This essentially completes the taxonomy of online $\etoe$-learning of Littlestone classes. Along the way, the right-hand side of \ref{eq:lit-vc-bound} resolves an open question of \cite{joshi2025theory}, who proved that $\VC(\cF^{\etoe-\iter}) = O(\LD(\cF) \log \iter)$ for all $\cF, \iter$, and asked if this upper bound is tight for some classes.

Before we formally state the results, we define precisely what we mean by a ``well-behaved" sub-logarithmic growth rate. Fix a sufficiently large universal constant $\iter_0$, to be used throughout the section. We say that a function $r:\mathbb{N}_+ \to \mathbb{N}_+$ is a \emph{well-behaved sub-logarithmic growth rate} if:
\begin{enumerate}
    \item $r$ is sub-logarithmic: $r(\iter) \leq \frac{1}{4} \log (\iter)$ for all $\iter \geq \iter_0$.
    \item $r$ is non-decreasing : $\iter_1 \geq \iter_2 \implies r(\iter_1) \geq r(\iter_2)$ for all $\iter_1,\iter_2 \geq 1$.
    \item $r$ is sub-additive: $r(\iter_1 + \iter_2) \leq r(\iter_1) + r(\iter_2)$ for all $\iter_1,\iter_2 \geq 1$.
\end{enumerate}

Note that all natural sub-logarithmic growth rates such as $\ceil*{\sqrt{\log \iter}}$, $\ceil*{\log \log \iter}$ etc.\ satisfy our definition. We can now state the result for well-behaved sub-logarithmic rates.

\begin{theorem} \label{thm:e2e-taxonomy}
    For any well-behaved sub-logarithmic growth rate $r$ there exists a base class $\cF$ such that:
    \begin{enumerate}
        \item $\LD(\cF) = 1$.
        \item $\LD(\cF^{\etoe-\iter}) = \Theta(r(\iter))$ for all $\iter \geq \iter_0$.
    \end{enumerate} 
\end{theorem}
The definition of $\cF$ and the lower bound strategy of the second item is due to unpublished private communication \cite{hanneke2026private}.
For logarithmic (in contrast to sub-logarithmic) rates, we prove the following

\begin{theorem} \label{thm:lit-logarithmic-lower-bound}
    For any $d \in \mathbb{N}_+$, there exists a class $\cF$ such that $\LD(\cF) = O(d)$ and 
    \[
    \VC(\cF^{\etoe-\iter}) = \Omega(\LD(\cF) \log \iter)
    \]
    for all $\iter \geq \iter_0$.
\end{theorem}

Together with the upper bound of Theorem~\ref{thm:e2e-upper-bound}, this establishes that for any $d \in \mathbb{N}_+$ there exists a class $\cF$ such that $\LD(\cF) = O(d)$ and
\[
\VC(\cF^{\etoe-\iter}), \LD(\cF^{\etoe-\iter}) = \Theta(\LD(\cF) \log \iter)
\]
for all $\iter \geq \max\{M_0,\LD(\cF)\}$.

In the following sections, we prove Theorem~\ref{thm:e2e-taxonomy} and Theorem~\ref{thm:lit-logarithmic-lower-bound}.

\subsection{Proof of Theorem~\ref{thm:e2e-taxonomy}}

Fix the growth rate $r$.

We first define the class $\cF$.
For every $s \geq \iter_0$ let
\[
K_s := \{r(s), r(s)+1, \ldots, r(s) + 2^{r(s)} - 1\}.
\]
That is, $K_s$ is simply the set $\{0, \ldots, 2^{r(s)}-1\}$ of size $2^{r(s)}$, with $r(s)$ added to each element.
We now define a ``baseline" function $f: \{0,1\}^\star \to \{0,1\}$ as follows. For every $s\geq \iter_0$, $k \in K_s$, $j \in \{0, \ldots, r(s) -1\}$ and $z \in \{0,1\}^j$, set:
\[
f(0^s 1 0^k 1 0^{s-k-1} z) = (k-r(s))_{j+1}
\]
where $(k-r(s))_{j+1}$ is the $(j+1)^{st}$ bit in the binary representation of $k-r(s)$.
For any other string $x$ that does not fit the form above, we set $f(x) = 0$. Note that $0^{s-k-1}$ is well-defined, since
\begin{equation} \label{eq:upper-bound-K_s}
    k+1 \leq r(s) + 2^{r(s)} \leq \frac{1}{4} \log s + 2^{\frac{1}{4} \log s} \leq \sqrt{s} \leq s-1,
\end{equation}
where the final inequality holds since $s \geq \iter_0$. Thus $s-k-1 \geq 1$.
Now, we define for each fixed pair $s$ and $k \in K_s$ the following single-point modification of $f$:
\[
f_{s,k}(x)
=
\begin{cases}
    1    & x=0^s 1 0^k, \\
    f(x) & \text{otherwise}.
\end{cases}
\]
We may now define the class $\cF$ as
\[
\cF := \{f_{s,k}: s\geq \iter_0, k \in K_s\}.
\]

We first prove the first item of Lemma~\ref{thm:e2e-taxonomy}.

\begin{lemma} \label{lem:e2e-taxonomy-lit-F}
    We have $\LD(\cF) = 1$.
\end{lemma}

\begin{proof}
    We have $\LD(\cF) \geq 1$ since $\cF$ contains more than one function. To see that $\LD(\cF) \leq 1$, consider the following online learner for $\cF$. The learner predicts as $f$ until it makes its first mistake. Once it makes a mistake, it discovers the correct function and always predicts as this function.
\end{proof}

We now prove the lower bound in the second item of Theorem~\ref{thm:e2e-taxonomy}. We prove it by showing the stronger claim $\VC(\cF^{\etoe-\iter}) \geq r(\iter)$ for all $\iter \geq \iter_0$.  

\begin{lemma} \label{lem:e2e-plausible-lower-bound}
    Fix $\iter \geq \iter_0$. Then the set
    \[
    A_{\iter} := \{a_i := 0^\iter 1 0^i: 1 \leq i \leq r(\iter)\}
    \]
    is shattered by $\cF^{\etoe-\iter}$.
\end{lemma}

\begin{proof}
    Let $\boldsymbol{y} = y_{1}, \ldots y_{r(\iter)}$ be a labeling of $A_\iter$. It suffices to show that this labeling is realized by $\cF^{\etoe-\iter}$. Let $b$ be the number given in binary representation by the vector $\boldsymbol{y}$. and let $k := r(\iter) + b \in K_\iter$. We claim that for all $i$, $f_{\iter,k}^{\etoe-\iter}(a_i) = b_i$. Let's describe the $\iter$-length computation path of $f^{\ct -\iter}_{\iter,k}$ on $a_i$:
    \begin{enumerate}
        \item In the first $k-i \geq b$ steps, $f_{\iter,k}$ agrees with $f$, which means that $0$ is appended.
        \item After the first $k-i \leq \sqrt{\iter}$ steps (the inequality is by \eqref{eq:upper-bound-K_s}), the computation path reaches the instance $0^\iter 10^k$, so in step  $k-i+1$, $1$ is appended, by definition of $f_{\iter,k}$.
        \item By definition of $f$, $0$ is appended in the following $\iter-k-1$ steps.
        \item So far, we completed a computation of length $k-i +1 + \iter - k -1 = \iter-i$, so we have $i$ more steps. By definition of $f$, in the last $i$ steps of the computation, the generated bits are $\boldsymbol{y}_{\leq i}$.
    \end{enumerate}
    Overall, we deduce that:
    \[
    f_{\iter,k}^{\ct-\iter}(a_i) = 0^{k-i} 1 0^{\iter-k-1} y_1, \ldots, y_i,
    \]
    which implies $f_{\iter,k}^{\etoe-\iter}(a_i) = y_i$.
\end{proof}

Lemma~\ref{lem:e2e-plausible-lower-bound} immediately implies the lower bound of the first item of Theorem~\ref{thm:e2e-taxonomy}. We now turn to the upper bound.

For every $s \geq \iter_0$ define the ``bucket"
\[
B_s:= \{0^s 1 0 ^i: i \geq 0\}.
\]

\begin{lemma} \label{lem:e2e-taxonomy-upper-bound-first-mistake-help}
    Fix $\iter, s \geq \iter_0$ and $k \in K_s$. For all $x \notin B_S$:
    \[
    f_{s,k}^{\etoe-\iter}(x) = f^{\etoe-\iter}(x).
    \]
\end{lemma}

\begin{proof}
    Since $f_{s,k}$ and $f$ differ only on the instance $x_{s,k} := 0^s 1 0^k$, it is possible that $f_{s,k}^{\etoe-\iter}(x) \neq f^{\etoe-\iter}(x)$ only if $x_{s,k}$ is an instance appearing in an intermediate step in the generation of $f^{\ct-\iter}(x)$. This is possible only if $x$ is a prefix of $x_{s,k}$. So, we only need to analyze the case where $x$ is both a prefix of $x_{s,k}$ and not in the bucket $B_s$. So the form of $x$ must be $x = 0^n$ for some $1 \leq n \leq s$. In this case, $f$ appends $0$ in every step, so $x_{s,k}$ never appears.
\end{proof}




We now prove another lemma that handles the case  $s \gg \iter$, where $B_s$ is the ``correct" bucket. The lemma establishes that in this case, for any function in $\cF^{\etoe-\iter}$ there is at most a single instance it labels with $1$.

\begin{lemma} \label{lem:e2e-taxonomy-upper-bound-second mistake}
    Fix $s \geq \iter_0$, $s \geq 10 \iter$, $i \geq 0$ and $k \in K_s$. Then
    \[
    f_{s,k}^{\etoe-\iter}(0^s 1 0^i) = 1 \iff \iter = k-i+1.
    \]
\end{lemma}

\begin{proof}
    For every point of the form in the statement of the lemma, $f$ appends $0$, Therefore, it is only possible that $f_{s,k}(0^s 1 0^j) = 1$ in the case where $f_{s,k}$ differs from $f$, that is, in the case $j=k$. Therefore, if $i > k$ then the instance $0^s 1 0^k$ is never reached in the autoregressive generation. Thus, we assume that $i \leq k$. We consider the three possible sub-cases.
    \begin{enumerate}
        \item If $\iter < k -i + 1 $, then $0$ the autoregressive function of $f_{s,k}$ appends $0$ for at most $k-i$ many times, and thus $f_{s,k}$ is never receives $0^s 1 0^k$ as input, and so $f_{s,k}^{\etoe-\iter}(0^s 1 0^i) = 0$.
        \item If $\iter = k -i + 1$, then $f_{s,k}^{\etoe-\iter}(0^s 1 0^i) = 1$, since in the last step the instance $0^s 1 0^i$ is given as input to $f_{s,k}$.
        \item If $\iter > k -i + 1$, then the instance $0^s 1 0^k$ is given as input to $f_{s,k}$, and then there is at least one more step before the autoregressive generation of $\ct$ is done, since  $\iter > k -i + 1$. The definition of $f$ implies that the following $s-k-1$ bits after $0^s 1 0^k 1$ are $0$, and only then $1$ could be appended. However, we have
        \[
        s - k -1
        \geq
        s - 2^{\frac{1}{2} \log s} -1 \geq s- \sqrt{s} - 1 \geq s/2 \geq 5\iter > \iter,
        \]
        where the first inequality is since $k \leq r(s) + 2^{r(s)}$ and the assumption $ r(s) \leq \frac{1}{4} \log s$, and the third inequality holds since $s \geq \iter_0$. So, only zeroes are appended, and thus $f_{s,k}^{\etoe-\iter}(0^s 1 0^i) = 0$.
    \end{enumerate}
\end{proof}

\begin{lemma}\label{lem:e2e-taxonomy-upper-bound}
    For all $\iter \geq \iter_0$:
    \[
    \LD(\cF^{\etoe-\iter}) \leq 20 r(\iter).
    \]
\end{lemma}

\begin{proof}
    We construct an online learner that makes at most $20 r(\iter)$ many mistakes.
    
    As long as the learner hasn't made a single mistake, it predicts as $f^{\etoe-\iter}$ on every instance. So, the number of mistakes made in this phase is at most $1$.
    
    After the first mistake on an instance $x$, the learner changes its behavior as follows. Let $f_{s^\star, k^\star}$ be the (unknown) target concept. So, we know that $f_{s^\star, k^\star}^{\etoe-\iter}(x) \neq f^{\etoe-\iter}(x)$. Therefore, Lemma~\ref{lem:e2e-taxonomy-upper-bound-first-mistake-help} implies that $x \in B_{s^\star}$. Thus $x = 0^{s^\star} 1 0^i $ for some $i \in \mathbb{N}$, so the learner now knows $s^\star$ by just looking at $x$. Thus, from this point, the learner only needs to learn the class
    \[
    \cF_{s^\star} := \{f_{s^\star, k^\star} : k \in K_{s^\star}\}.
    \]
    We consider the two possible cases.

    First, suppose that $s^\star \leq 10 \iter$. In this case, the learner runs the halving algorithm over $\cF_{s^\star}$. The mistake bound in this case is
    \[
    \log | \cF_{s^\star}| = \log 2^{r(s)} = r(s) \leq r(10 \iter) \leq 10 r(\iter),
    \]
    where the first inequality is since $r$ is non-decreasing, and the second inequality is since $r$ is sub-additive.

    In the complementing case, we have $s^\star > 10 \iter$. As usual, for $x \notin B_{s^\star}$, the learner predicts as $f$, and thus is always correct for $x \notin B_{s^\star}$. For $x\in B_{s^\star}$, that is, $x$ of the form $0^{s^\star}1 0^i$, the learner always predicts $0$, until making a mistake. Once it makes a mistake on the instance $0^{s^\star}1 0^i$ for some $i \in \mathbb{N}$, it means that $f_{s^\star, k^\star}^{\etoe-\iter}(0^{s^\star} 1 0^i) = 1$.  Lemma~\ref{lem:e2e-taxonomy-upper-bound-second mistake} implies that in this case $k^\star = \iter + i -1$. So the learner now knows both $s^\star$ and $k^\star$, and therefore knows the target concept $f_{s^\star, k^\star}$ and can always predict as $f_{s^\star, k^\star}$.

    Summing all mistakes, we deduce that overall the learner makes at most $1 + 10 r(\iter)$ many mistakes, which concludes the proof.
\end{proof}

\begin{proof}[Proof of Theorem~\ref{thm:e2e-taxonomy}]
    The theorem is given by combining Lemma~\ref{lem:e2e-taxonomy-lit-F}, Lemma~\ref{lem:e2e-plausible-lower-bound}, and Lemma~\ref{lem:e2e-taxonomy-upper-bound}.
\end{proof}

\subsection{Proof of Theorem~\ref{thm:lit-logarithmic-lower-bound}}

We show that for any $d \in \mathbb{N}$, there exists a class $\cF$ with $\LD(\cF) = O(d)$, and such that $\VC(\cF^{\etoe-\iter}) = \Omega(d \log \iter)$ for all sufficiently large $\iter$.  Our general upper bound asserts that $\LD(\cF^{\etoe-\iter}) = O(d \log \iter)$ for all $\iter \geq \LD(\cF)$, and thus $ \LD(\cF^{\etoe-\iter}) = \Theta(d \log \iter)$.

Our construction solves an open question from \cite{joshi2025theory}. They showed that for any class $\cF$ and all $\iter$, it holds that $\VC(\cF^{\etoe-\iter}) = O(\LD(\cF) \log \iter)$, and asked if the logarithmic dependence (or, in fact, any dependence) on $\iter$ is necessary. Our proof gives a positive answer to this question.

The proof is via the probabilistic method applied to a ``hard" random class defined below.

\subsubsection{The hard class} \label{sec:hard-class}

Fix $d,\iter \in \mathbb{N}_+$. Let
\[
N := \iter^{10 d}, \quad Z:= \{0^i: i \in [\iter^2]\}.
\]
be the size of the class, and $Z$ be the set of instances on which the class is ``interesting". Partition $Z$ to
\[
Z_1 := \{0^{j \iter: j \in [\iter]}\}, \quad Z_0 := Z \backslash Z_1,
\]
so $|Z_1| = \iter, |Z_0| = \iter^2 - \iter$.
We now draw $N$ functions independently at random. For each $a\in [N]$ and $x\in \{0,1\}^\star$, define $f_a(x)$ as follows:
\begin{itemize}
    \item If $x \notin Z$, let $f_a(x) = 0$.
    \item Otherwise, if $x\in Z_1$, let $f_a(x) = 1$ w.p. $1-1/\iter$, and $f_a(x) = 0$ w.p. $1/\iter$.
    \item Otherwise, if $x\in Z_0$, let $f_a(x) = 0$ w.p. $1-1/\iter$, and $f_a(x) = 1$ w.p. $1/\iter$.
\end{itemize}
Now, let
\[
\cF_{d,\iter} := \{f_a: a\in [N]\}.
\]
For each $x \in Z_0$, we say that $1$ is the \emph{minority} label of $x$, and $0$ is the \emph{majority} label of $x$. We define the analog for $Z_1$.

We are now ready to ready to begin with the proof of Theorem~\ref{thm:lit-logarithmic-lower-bound}. We begin by proving the upper bound $\LD(\cF_{d, \iter}) = O(d)$ with high probability.

\subsubsection{Upper bounding the Littlestone dimension of the base class}

The intuition for this bound is that since $|\cF| = N = \iter^{10d}$ and for each instance there exists a label $r \in \{0,1\}$ for which typically only a $1/\iter$-fraction of the version agrees with, the version space should be empty at the bottom of some branch of length $O(d)$. However, this is only true on \emph{average}, while the Littlestone dimension is determined by the \emph{maximal} depth of a shattered tree, so the proof is in fact more challenging than this intuition. We use the technique of \cite{chase2024deterministic}, who provided such an analysis for a similar random class.

\begin{lemma}\label{lem:upper-bound-base}
    Let $d, \iter \in \mathbb{N}$ larger than some large enough universal constant.  With probability at least $3/4$, we have
    \[
    \LD(\cF_{d,\iter}) \leq 100 d.
    \]
\end{lemma}

\begin{proof}
    Denote $\cF:=\cF_{d,\iter}$. We show that the probability that $\cF$ shatters a tree of depth $100d$ is at most $1/4$. Fix a tree $\tree$ whose internal nodes are labeled by instances from $\{0,1\}^\star$, and that is shattered by $\cF$. If $\tree$ contains nodes labeled with instances outside of $Z$ it cannot be shattered, so we may assume that all nodes are labeled by instances from $Z$. So, every node has precisely one \emph{minority edge}, corresponding to the minority label of the instance labeling it. Let $B_{1/2}(\tree)$ be the set of branches in $\tree$ in which precisely $50d$ of the edges are minority edges. Note that
    \begin{equation}\label{eq:base-branches-size}
        \mleft \lvert B_{1/2}(\tree) \mright \rvert
        =
        \binom{100 d}{50 d}
        \geq
        \frac{2^{100d}}{2 \sqrt{50d}}
        \geq
        \frac{2^{100d}}{20 \sqrt{d}},
    \end{equation}
    where the first inequality is by the known bound $\binom{2n}{n} \geq \frac{2^{2n}}{2 \sqrt{n}}$, applied with $n = 50d$.
    Since $\tree$ is shattered, in particular every branch of $B_{1/2}(\tree)$ is realized by at least one function $f\in \cF$. Fix a branch $b \in B_{1/2}(\tree)$ and $f \in \cF$. Since the draw of $f$ is independent for each instance
    \begin{equation} \label{eq:base-fixed-b-f}
        \Pr[f \text{ realizes } b] \leq \frac{1}{\iter ^{50d}} .
    \end{equation}
    Now, note that each function can realize at most a single branch. Therefore, if $\tree$ is shattered, then there exists an injective mapping $\phi: B_{1/2}(\tree) \to [N]$ such that for all $b \in B_{1/2}(\tree)$, $f_{\phi(b)}$ realizes $b$. That is $\phi$ assigns to each $b \in B_{1/2}(\tree)$ the index of the function realizing it, and the same function cannot be assigned to two different branches. Fix a mapping $\phi$. Since the probability that $f_{\phi(b)}$ realizes $b$ is independent of the other randomness, and using \eqref{eq:base-fixed-b-f}, the probability that for all $b \in B_{1/2}(\tree)$, $f_{\phi(b)}$ realizes $b$ is at most
    \begin{equation} \label{eq:base-all-branches-realized}
        \mleft( \frac{1}{\iter ^{50d}} \mright)^{|B_{1/2}(\tree)|}.
    \end{equation}

Since there are at most $N^{|B_{1/2}(\tree)|} = \iter^{10d |B_{1/2}(\tree)|} $, using \eqref{eq:base-all-branches-realized} and \eqref{eq:base-branches-size}, the probability that there exists $\phi$ so that for all $b \in B_{1/2}(\tree)$, $f_{\phi(b)}$ realizes $b$ is at most
\begin{equation} \label{eq:base-fixed-tree}
    \iter^{10d |B_{1/2}(\tree)|} \mleft( \frac{1}{\iter ^{50d}} \mright)^{|B_{1/2}(\tree)|}
    =
    \mleft( \frac{1}{\iter ^{40d}} \mright)^{|B_{1/2}(\tree)|}
    \leq
    \mleft( \frac{1}{\iter ^{40d}} \mright)^{\frac{2^{100d}}{20 \sqrt{d}}}
    \leq
    \frac{1}{\iter ^{2 \sqrt{d}\cdot 2^{100d}}}.
\end{equation}

So, we deduce that for a fixed tree $\tree$, the probability that it is shattered by $\cF$ is at most $\frac{1}{\iter ^{2 \sqrt{d}\cdot 2^{100d}}}$. The number of perfect trees of depth $100d$ labeled by instances from $Z$ is at most $|Z|^{2^{100d}} = \iter^{2\cdot 2^{100d}}$. Thus, using \ref{eq:base-fixed-tree},  the probability that there exists a tree $\tree$ of depth $100d$ which is shattered by  $\cF$ is at most
\[
\iter^{2\cdot 2^{100d}} \frac{1}{\iter ^{2 \sqrt{d}\cdot 2^{100d}}}
=
\iter^{2\cdot 2^{100d} (1 - \sqrt{d})}
<
1/4,
\]
which concludes the proof.
\end{proof}

We now turn to the lower bound $\VC(\cF^{\etoe-\iter}) = \Omega(d \log \iter)$ with high probability. To formalize a choice of a random subset before it is drawn, we use standard propositional logic formalization, as defined below.

\subsubsection{A framework for choosing subsets from a random class}

We define a notion of \emph{rules} for choosing a subset of a random class $\cF$ before it is drawn, by restricting the labels of some the instances in the domain. A rule is any restriction on the labels of some of the instances, that can be represented as a propositional logic formula with the conjunction and disjunction connectives, where the examples $(x,y) \in \cX \times \cY$ are the atomic variables, and the assignments are given by the functions of the class. A rule is said to be \emph{independent} of an instance $x$, if $x$ does not appear in any of the atomic variables used in the rule. For example, if $x_1,x_2,x_3,x_4 \in \cX$, then $((x_1,0) \land (x_2,1)) \lor (x_3, 0)$ is a rule, and it is independent of $x_4$. The \emph{truth value} of a variable (example) $(x,y)$ under the assignment (function) $f \in \{0,1\}^\cX$ is naturally defined as $f(x,y) := 1[f(x) = y]$, where $f(x,y) = 1$ represents \emph{true} and $f(x,y) = 0$ \emph{false}. The truth value of a rule $R$ under the function $f$, denoted by $f(R)$, is inductively defined in the standard way used in propositional logic. A rule $R$ for the class $\cF$ naturally defines a subclass of $\cF$, denoted as $\cF(R)$, of all functions satisfying the rule:
\[
\cF(R) = \{f \in \cF : f(R) = 1 \}.
\]

Given a decision tree, any branch $b$ in the tree naturally induces the rule $(\boldsymbol{x}(b), \boldsymbol{y}(b))$ (formally, we need to replace the commas in the sequence $(\boldsymbol{x}(b), \boldsymbol{y}(b))$ to $\land$ connectives, to make it a valid rule, but it is understood that $(\boldsymbol{x}(b), \boldsymbol{y}(b))$ represents this rule). Likewise, any two branches $b_1,b_2$ naturally induce the rule $(\boldsymbol{x}(b_1), \boldsymbol{y}(b_1)) \lor (\boldsymbol{x}(b_2), \boldsymbol{y}(b_2))$, and this can be extended to any number of branches. For a set of branches $B$, we let $R(B)$ be the rule induced by $B$.

We now use this framework and prove a standard concentration lemma for predetermined rules defined over $\cF_{d,\iter}$.

\subsubsection{A concentration lemma for predetermined rules}

\begin{lemma}\label{lem:conentration}
    Let $d \in \mathbb{N}$ larger than some large enough universal constant, and let $\iter \geq d^2$. Let $\cR$ be a family of rules such that $|\cR| \leq e^{\iter^{6d}}$. Then w.p.\ at least $3/4$ over the draw of $\cF_{d, \iter}$, for all $x \in Z_r$ (where $r \in \{0,1\}$) and all $R \in \cR$, such that $R$ is independent of $x$ and $|\cF_{d,\iter}(R)| \geq \iter^{8d}$, we have
    \[
     \frac{1}{2\iter} |\cF_{d, \iter} (R)| \leq |\cF_{d, \iter}( R \land (x,1-r))| \leq \frac{3}{2\iter} |\cF_{d, \iter} (R)|.
    \]
\end{lemma}

\begin{proof}
    Since $R$ is independent of $x$, we have
    \[
    X := |\cF_{d, \iter}( R \land (x,1-r))| \sim \Bin(|\cF_{d, \iter}(R)|, 1/\iter),
    \quad
    \mu := \mathbb{E}[X] = |\cF_{d, \iter}(R)|/\iter \geq \iter^{8d-1}.
    \]
    Multiplicative Chernoff bound thus gives
    \[
    \Pr \mleft[X \notin \mleft[\frac{\mu}{2}, \frac{3\mu}{2} \mright] \mright] \leq 2 e^{-\mu/12} \leq e^{-\iter^{8d-1}/12}.
    \]
    Taking a union bound over all $x \in Z, R \in \cR$, the failure probability that there exist a pair $(x,R)$ not satisfying the lemma is at most
    \[
    |Z| |\cR| e^{-\iter^{8d-1}/12} \leq \iter^2 e^{\iter^{6d}} e^{-\iter^{8d-1}/12} < 1/4,
    \]
    where the final inequality holds for large enough $d$.
\end{proof}

\subsubsection{Lower bounding the VC-dimension of the end-to-end class}

We are now ready to prove the lower bound on $\VC(\cF_{d,\iter}^{\etoe,\iter})$. In light of Lemma~\ref{lem:conentration}, we fix a class $\cF:=\cF_{d,T}$ drawn as in Section~\ref{sec:hard-class}, for which the lemma holds.

\begin{lemma} \label{lem:iterated-lower-bound-fixed-T}
    We have
    \[
    \VC(\cF^{\etoe,\iter}) = \Omega(d \log \iter).
    \]
\end{lemma}

Let $m := \floor*{\frac{d \log \iter}{600}}$. For each $j \in [m]$, let $x^{(j)} := 0^{(j-1)\iter +1}$, and denote $S: = \{x^{(j)} : j \in [m]\}$. We will prove that $S$ is shattered by $\cF^{\etoe,\iter}$.

The main tool used to prove Lemma~\ref{lem:iterated-lower-bound-fixed-T} is the concentration lemma proved above. To use it, we need to define set $\cR$ of rules determined before drawing $\cF_{d,\iter}$. We will need two different types of rules, one to control the version space, and one to control any possible specific autoregressive trajectory. Let us start with the first type. 
For each $j \in [m]$, consider the generation tree of depth $\iter$ rooted at $x^{(j)}$. In this tree, define
\begin{itemize}
    \item The \emph{green branch}, given by the sequence $g^{(j)} := 0^{T-1} 1$ of edge labels.
    \item The \emph{red branches}, given by the sequences 
    \[
    r^{(j)}_q := 0^{q-1} 1 0^{\iter-q}, \quad q \in [\iter - 1]. 
    \]
\end{itemize}

See Figure~\ref{fig:hard-class} for an illustration of the green and red branches, as well as an intuitive explanation for how they can be used to yield the lower bound on $\VC(\cF_{d,\iter}^{\etoe,\iter})$.

Now, let
\[
G^{(j)} := \{g^{(j)}\}, \quad R^{(j)} := \{r^{(j)}_q: q \in [\iter-1]\}
\]
be the sets of green and red branches, respectively. We now define for every possible labeling $u:= (u_1, \ldots, u_t) \in \{0,1\}^t$ of $(x^{(1)}, \ldots, x^{(t)})$ where $t \leq m$, a rule $Q_{u}$ as follows:
\[
Q_u := \bigwedge_{j \in [t]: u_j =1} R(G^{(j)}) \land \bigwedge_{j \in [t]: u_j = 0} R(R^{(j)}).
\]

That is, $Q_{u}$ requires to satisfy the rule induced by the green branch $g^{(j)}$ for $x^{(j)}$ that are labeled by $1$, and to satisfy the rule induced by the red branches $R^{(j)}$ for $x^{(j)}$ that are labeled by $0$. This makes sense, as one can note that
\begin{itemize}
    \item $f \models R(G^{(j)}) \implies f^{\etoe-\iter}(x^{(j)}) = 1$.
    \item $f \models R(\{r^{(j)}_q\}) \implies f^{\etoe-\iter}(x^{(j)}) = 0$, for any $q \in [\iter-1]$.
\end{itemize}

Now, let
\[
\cR_{ver} := \{Q_u : u \in \{0,1\}^t, t \in [m]\}.
\]

We now turn to the second type. For two instances $x,y \in \{0,1\}^\star$ and $t \leq |y|$, we define the appropriate rule for the $t$-prefix path that begins at $x$ and follows the trajectory given by the first $t$ bits of $y$ as
\[
P_t(x,y) := \bigwedge_{s=1}^t (x b_1 \ldots b_{s-1}, b_s).
\]
In the case where we consider the entire trajectory of $y$, that is $t = |y|$, we abbreviate $P(x, y) := P_{|y|}(x,y)$. So, $f \models P_{t}(x,y)$ means that when starting from the instance $x$, the first $t$ bits autoregressivly generated by $f$ are precisely $y_1, \ldots, y_t$. Now we can also define a rule that captures functions from a specific version space given by a labeling $u \in \{0,1\}^{< m}$, that also follow a specific trajectory of a green or red branch starting at the next instance $x^{(|u|+1)}$ . Formally, let:
\[
\mathcal{\cR}_{bran} := \{Q_u \land P_t(x^{(|u|+1), b}): u \in \{0,1\}^{< m}, t \in [T], b \in G^{(|u|+1)} \cup R^{(|u|+1)}\}.
\]

Now, let
\[
\cR := \cR_{ver} \cup \cR_{bran}
\]

To use Lemma~\ref{lem:conentration}, we need to show that $\cR$ is not too large.

\begin{lemma}
    We have
    \[
    |\cR| \leq \iter^{2d}
    \]
\end{lemma}

\begin{proof}
    We have
    \[
    |\cR_{ver}| \leq m 2^m,
    \]
    since there are at most $m 2^m$ options for the choice of $u$.
    directly from the definition. Note that for all $ j\in [m]$, $|G^{(j)} \cup R^{(j)}| \leq \iter$. Thus
    \[
    |\cR_{bran} | \leq m 2^m \cdot \iter\cdot \iter
    \]
    since there are at most $m 2^m$ options for the choice of $u$, at most $T$ options for the choice of $t$, and at most $\iter$ options for the choice of $b$.
    So, over all:
    \[
    |\cR| \leq |\cR_{ver}| + |\cR_{bran} | \leq m 2^m (\iter^2 + 1) \leq d \log \iter \cdot \iter^d (\iter^2 + 1) \leq \iter^{2d},
    \]
    which concludes the proof.
\end{proof}

We now show that when committing to a label of another instance, the version space does not decrease by too much.

\begin{lemma} \label{lem:lit-lower-bound-inductive-step}
    Let $u \in \{0,1\}^t$ where $t < m$, and suppose that $|\cF(Q_u)| \geq \frac{1}{e^{300|u|}} N $. Then for any $r \in \{0,1\}$:
    \[
    |\cF(Q_{ur})| \geq \frac{1}{e^{300}} |\cF(Q_u)|.
    \]
\end{lemma}

\begin{proof}
    To use Lemma~\ref{lem:conentration}, we have to show that any rule $R$ we use, it holds that $|\cF(R)| \geq \iter^{8d}$. We will use rules of the form
    \begin{equation} \label{eq:inductive-step-rules}
        Q_u \land P_t(x^{|u|+1}, b),
    \end{equation}
    which are indeed in $\cR_{bran} \subset \cR$.
    where $t \in [\iter]$, and $b \in G^{(|u|+1)} \cup R^{(|u|+1)}$. By assumption,
    \[
    |\cF(Q_u)| \geq \frac{1}{e^{300|u|}} N \geq \frac{1}{\iter^d} \iter^{10 d} = \iter^{9d},
    \]
    so the bound of Lemma~\ref{lem:conentration} holds for $Q_u$. Now, from monotonicity of the inclusion relation, to prove that for each possible $P_t(x^{|u|+1}, b)$ we have $\cF(Q_u \land P_t(x^{|u|+1}, b)) > \iter^{8d}$, it suffices to show that it holds for the worst-case $P_t(x^{|u|+1}, b)$. In the worst case, $b$ has a single minority edge and $T-1$ majority edges. Thus, using $(1-100/\iter)^{\iter} \geq e^{-200}$ for large enough $\iter$, for any rule of the form of \eqref{eq:inductive-step-rules} we have
    \[
    \cF(Q_u \land P_t(x^{|u|+1}, b))
    \geq
    \frac{1}{100 \iter}(1-100/ \iter)^{\iter} |\cF(Q_u)|
    \geq
    \frac{1}{100 \iter} e^{-200} \iter^{9d}
    \geq
    \iter^{9d-2},
    \]
    where the final inequality holds for large enough $\iter$.
    We can now prove the statement of the lemma. Let us start with the case $r=1$. Recall that any function $f$ realizing the green branch $g^{(|u|+1)}$ satisfies $f^{\etoe-\iter}(x^{(|u|+1)}) = 1$. So, it suffices to lower bound the number of such functions. Since the green branch has only majority edges, we obtain:
    \[
    |\cF(Q_{u1})| \geq (1-100/\iter)^{\iter} | \cF(Q_u)|
    \geq
    e^{-200} | \cF(Q_u)|.
    \]
    Likewise, any function $f$ realizing any of the red branches $r_q^{(|u|+1)}$ for some $q\in [\iter-1]$ satisfies $f^{\etoe-\iter}(x^{(|u|+1)}) = 0$.  So, it suffices to lower bound the number of such functions. Any red branch has a single minority edge, and there are $\iter-1$ red branches. Furthermore, a function can realize at most one branch, so the set of functions realizing different red branches are disjoint.  Thus overall:
    \[
    |\cF(Q_{u0})|
    \geq
    (\iter -1 ) \frac{1}{100 \iter} (1-100/\iter)^{\iter} | \cF(Q_u)|
    \geq
    e^{-300} | \cF(Q_u)|,
    \]
    where the final inequality holds for large enough $\iter$.
\end{proof}

We are now ready to prove Lemma~\ref{lem:iterated-lower-bound-fixed-T}.
\begin{proof}[Proof of Lemma~\ref{lem:iterated-lower-bound-fixed-T}]
    We prove the lemma by showing that $|\cF_(Q_u)| > 0$ for all $u$, which means that any labeling of $S$ is realized by some function in $\cF$. This establishes that $\VC(\cF_{d,\iter})^{\etoe,\iter} \geq m = \Omega(d \log \iter)$.
    The claim follows almost directly from Lemma~\ref{lem:lit-lower-bound-inductive-step} by proving the following stronger claim:
    \begin{equation}\label{eq:lower-bound-stronger-claim}
        |\cF(Q_u)| \geq \frac{1}{e^{300|u|}} N
    \end{equation}
    for all $u \in \{0,1\}^{\leq m}$. Indeed, having \eqref{eq:lower-bound-stronger-claim} we are done:
    \[
    \frac{1}{e^{300|u|}} N \geq \frac{1}{e^{300 \frac{d \log \iter}{600}}} \iter^{10d} \geq \iter^{9d}.
    \]
    So, it remains to prove \eqref{eq:lower-bound-stronger-claim}. We prove it by induction on $|u|$. To handle the base case $|u|=0$, let us use the notation $\cF(Q_0) = \cF$. That is $Q_0$ means just $\cF$ with no further restrictions.
    So indeed:
    \[
    |\cF(Q_u)| = |\cF| = N = \frac{1}{e^{300\cdot 0}} N = \frac{1}{e^{300\cdot |u|}} N,
    \]
    so the base case holds. For the induction step, Let $u' = ur$ where $|u'| \geq 0$ and $r \in \{0,1\}$. By induction hypothesis, we may assume that $|\cF(Q_u)| \geq \frac{1}{e^{300|u|}} N$  So:
    \[
    |\cF(Q_{u'})| = |\cF(Q_{ur})| \geq \frac{1}{e^{300}} \frac{1}{e^{300|u|}} N = \frac{1}{e^{300|u'|}} N
    \]
    where the inequality is by the induction hypothesis and Lemma~\ref{lem:lit-lower-bound-inductive-step}. This concludes the proof.
\end{proof}

Recall that $\cF:= \cF(d,\iter)$. That is, so far we assumed that the number of autoregressive iterations $\iter$ is fixed. The only remaining step is to prove that we can construct from $\cF(d,\iter)$ a single class that uniformly handles all $\iter$ values, without significantly increasing it's Littlestone dimension. We prove this final step below.

\subsubsection{A gluing lemma}
We prove a \emph{gluing lemma} that glues together the classes $\{\cF_{d,\iter}\}_{\iter \in \mathbb{N}_+}$ without increasing the Littlestone dimension by more than $1$. Our lemma is similar to the gluing lemma of \cite{chase2024deterministic}.

\begin{lemma} \label{lem:glue}
    Let $c > 0$ be a universal constant, and fix $d\in \mathbb{N}$ larger than some universal constant. Let $f_d:\mathbb{N} \to \mathbb{N}$ be a function. Suppose that for all large enough $\iter$ there exists a class $\cF_{d,\iter}$ such that:
    \begin{enumerate}
        \item $\LD(\cF_{d,\iter}) \leq c \cdot d$.
        \item $\VC(\cF_{d, \iter}^{\etoe-\iter}) \geq f_d(\iter)$
        \item Let
        \[
        X_\iter := \{x \in \{0,1\}^\star: \exists f \in \cF_{d,\iter}, f(x) = 1\}.
        \]
        Then $|X_\iter| \leq \iter^3$.
    \end{enumerate}
    Then, there exists a class $\hat{\cF}$ such that $\LD(\hat{\cF}) \leq c \cdot d + 1$ and for all large enough $\iter$ it holds that $\VC(\cF^{\etoe-\iter}) \geq f_d(\iter)$.
\end{lemma}

\begin{proof}
    For every $\iter \geq 1$ and for every $f \in \cF_{d,\iter}$, define the shifted function $\hat{f}$ of $f$ as
    \[
    \hat{f}(x')
    =
    \begin{cases}
        f(x) & \exists x \in \{0,1\}^\star: x' = 0^{2^{\iter^4}} x, \\
        0    & \text{otherwise,}
    \end{cases}
    \]
    for all $x' \in \{0,1\}^\star$. Now, let
    \[
    \hat{\cF}_{d, \iter} := \{\hat{f}: f \in \cF_{d,\iter}\}.
    \]
    Denote the analog of $X_{\iter}$ for $\hat{\cF}_{d, \iter}$ as
    \[
    \hat{X}_\iter :=  \{x \in \{0,1\}^\star: \exists \hat{f} \in \hat{\cF}_{d,\iter}, \hat{f}(x) = 1\}.
    \]
    We show that $\iter_1 > \iter_2 \implies \hat{X}_{\iter_1} \cap \hat{X}_{\iter_2} = \emptyset$.
    Indeed, the length of every instance in $\hat{X}_{\iter_2}$ is at most $2^{\iter^4_2} + \iter_2^3$. However, the length of every instance in $\hat{X}_{\iter_1}$ is at least
    \[
    2^{\iter^4_1} \geq 2^{(\iter_2 + 1)^4} \geq  2^{\iter_2^4} 2^{\iter_2^3} > 2^{\iter^4_2} + \iter_2^3,
    \]
    and so indeed $\hat{X}_{\iter_1} \cap \hat{X}_{\iter_2} = \emptyset$.
    Now, define
    \[
    \hat{\cF} := \bigcup_{\iter} \hat{\cF}_{d, \iter}.
    \]
    where the union is taken over all large enough $\iter$.
    First, since $\hat{\cF}$ contains $\hat{\cF}_{d, \iter}$ for all large enough $\iter$, it is clear that for all large enough $\iter$:
    \[
    \VC(\hat{\cF}^{\etoe-\iter}) \geq \VC(\hat{\cF}_{d, \iter}^{\etoe-\iter}) \geq f_d(\iter),
    \]
    as desired. It remains to prove that $\LD(\hat{\cF}) \leq c\cdot d + 1$. To see why this is true, consider an online learner for $\hat{\cF}$ that always predicts $0$, until it makes it first mistake. Once it makes a mistake, it can infer that the target function gives $1$ to the instance it erred for. Since $\hat{X}_{\iter_1} \cap \hat{X}_{\iter_2} = \emptyset$ for any two distinct $\iter_1,\iter_2$, this reveals the identity of the class $\hat{\cF}_{d,\iter}$ from which the target function is taken. Now, the learner learns $\hat{\cF}_{d,\iter}$, on which it will make at most $c \cdot d$ many mistakes by assumption. This concludes the proof.
\end{proof}

We may now prove the main theorem.
\begin{proof}[Proof of Theorem~\ref{thm:lit-logarithmic-lower-bound}]
    The theorem follows from combining Lemma~\ref{lem:iterated-lower-bound-fixed-T} and Lemma~\ref{lem:glue}.
\end{proof}

\section{Agnostic taxonomy} \label{sec:agnostic}

Consider the following known result.

\begin{theorem}[\cite{bendavid2009agnostic ,alon2021adversarial}] \label{thm:agnostic-optimal-regret}
    Let $\cF$ be a binary function class.\footnote{The upper bound of \cite{alon2021adversarial} requires that $\cF$ satisfies some mild measurability constraints, which hold for all classes relevant to the scope of this paper, since our domain $\cX = \{0,1\}^\star$ is countable.} Then the optimal regret bound for agnostic online learning of $\cF$ is
    \[
    \Theta \mleft(\sqrt{\LD(\cF) \rounds} \mright).
    \]
\end{theorem}





Using these bounds and our taxonomy of possible Littlestone dimension of $\cF^{\etoe-\iter}$, we may infer an essentially complete taxonomy for agnostic learning as well. First, we deduce a general upper bound.

\begin{corollary}\label{cor:e2e-agnostic-upper-bound}
    Let $\cF$ be a base class. Then for all sufficiently large $\iter$, there exists an agnostic online learner for $\cF^{\etoe-\iter}$ with regret bound
    \[
    O \mleft( \sqrt{T \LD(\cF) \log \iter} \mright).
    \]
\end{corollary}

Below, we show that all sub-logarithmic well-behaved rates between $\sqrt{T}$ and $\sqrt{ T \log \iter}$ are attainable. Note that due to Theorem~\ref{thm:agnostic-optimal-regret}, going below order of $\sqrt{T}$  regret is possible only in the case that the autoregressive process degenerates $\cF^{\etoe-\iter}$ to a single function. Therefore, this essentially establishes a complete taxonomy for agnostic online $\etoe$-learning.

\begin{corollary} \label{cor:e2e-agnostic-taxonomy}
    For any well-behaved sub-logarithmic growth rate $r$ there exists a base class $\cF$ such that $\LD(\cF) = 1$ and the optimal regret bound for agnostic online learning $\cF^{\etoe-\iter}$ is
    \[
    \Theta \mleft (\sqrt{\rounds \cdot r(\iter)} \mright)
    \]
    for all $\iter \geq \iter_0$ and sufficiently large $\rounds$.
\end{corollary}

A similar result holds also for logarithmic rates.

\begin{corollary}
    For any $d \in \mathbb{N}_+$, there exists a class $\cF$ such that the optimal regret for agnostic online learning of  $\cF^{\etoe-\iter}$ is
    \[
    \Theta \mleft (\sqrt{T d \log \iter} \mright)
    \]
    for all $\iter \geq \iter_0$ and sufficiently large $\rounds$.
\end{corollary}

\section{Learning with chain-of-thought supervision} \label{sec:cot}

In contrast to $\etoe$-learning, we show that similarly to the statistical setting (see \cite{hanneke2026sample}), the optimal mistake bound of learning a Littlestone class is independent of $\iter$.

\begin{theorem}\label{thm:cot-bound-lit}
    For every class $\cF$, and for every $\iter$, we have $\LD(\cF^{\ct-\iter})\leq \LD(\cF)$.
\end{theorem}

\begin{proof}
    We use an SOA-style learner who realizes the stated mistake bound, described below.
    The learner maintains a valid version space $V:=V_t$ of $\cF$, which in the first round is $V_1=\cF$. Let $x:=x_t$ be the instance given by the adversary in round $t$. The learner constructs the generation tree $\tree_{ \iter}(x)$. For a branch $b$, let $V_b$ be the version space $V$ restricted to $b$. That is, those are all functions in $V$ that agree with the branch $b$. We argue that there is at most one branch $b^\star$ so that $\LD(V_{b^\star}) = \LD(V)$. Indeed, suppose that there are two such branches $b_0,b_1$. Let $v$ be the internal node where these branches first split, and let $v_0,v_1$ be the left and right children of $v$, where $b_0$ goes through $v_0$, and $b_1$ goes through $v_1$. So, the version space $V$ restricted to both $b_0$ up to $v_0$ and $b_1$ up to $v_1$ has Littlestone dimension $\LD(V)$. Therefore, we can construct a tree of depth $\LD(V)+1$ which is rooted at $v$ and is shattered by $V$, which is a contradiction. So, we determine the learner's prediction in round $t$ to be the label associated with the edges of the branch
    \[
    b^\star := \argmax_{b \in B(\tree_\iter(x))} \LD(V_b).
    \]
    Now, if the learner errs and the correct label is $b \neq b^\star$, then the learner updates $V_{t+1} = V_b$, and $\LD(V_{t+1}) < \LD(V_t)$. Telescoping over all rounds where the learner errs, at most $\LD(\cF)$ many mistakes are made before $\LD(V_t) = 0$, and then a single consistent function remains, and the learner's predictions are trivially given by this function from this round.
\end{proof}

Note that in contrast with the $\etoe$ setting, here we cannot infer an agnostic regret bound from the realizable mistake bound. Indeed, since the learner of Theorem~\ref{thm:cot-bound-lit} learns the entire $\ct$, every change in the $\ct$ counts as a mistake. Thus, it is possible that for some function in $\cF$ the $\etoe$-loss is extremely good, while for all functions in $\cF$ the $\ct$-loss is very bad. In such case, the learner of Theorem~\ref{thm:cot-bound-lit} is measured with respect to the function with the best $\ct$-loss, its regret might be very bad when measured against the function with the best $\etoe$-loss. It is an interesting direction for future research to prove that this obstacle can or cannot be circumvented by some learner.

Using an online-to-batch conversion (Theorem~\ref{thm:online-to-pac}), Theorem~\ref{thm:cot-bound-lit} immediately allows to derive a PAC sample complexity bound for learning $\cF^{\ct-\iter}$.

\begin{theorem} [Upper bound in terms of $\LD(\cF)$] \label{thm:pac-cot-bound-lit}
    For every base class $\cF \subset \Sigma^{\Sigma^\star}$ (even if $|\Sigma| > 2$, and even for $|\Sigma| = \infty$), and for all $\iter$, if $\LD(\cF) < \infty$ then $\cF^{\ct-\iter}$ is learnable with sample complexity
    \[
    m(\eps, \delta) = O \mleft( \frac{\LD(\cF) + \log (1/\delta)}{\eps} \mright).
    \]
\end{theorem}

\section{Autoregressive Linear classifiers}

In this section, we study the base class of autoregressive linear classifiers in $\mathbb{R}^d$. It is defined to only take into account the final $d$ entries. Formally, define

\[
\cF_{d, \lin} := \{f_{\boldsymbol{w}, b}: \boldsymbol{w}\in \mathbb{R}^d, b\in \mathbb{R}\}
\]
where the function $f_{\boldsymbol{w}, b}$ is given as
\[
f_{\boldsymbol{w}, b}(x) = 1 \mleft[ \sum_{i=1}^{\min\{d, |\boldsymbol{x}|\}} \boldsymbol{w}[-i] x[-i] + b \geq 0 \mright]
\]
for all $x \in \Sigma^\star$.

\subsection{Mistake and sample complexity bounds} \label{sec:linear-mistake}

$\cF_{d, \lin}$ is essentially the class of linear classifiers over the $d$-cube. This immediately gives the following bound

\begin{proposition}
    For all $\iter$:
    \[
    \LD(\cF_{d, \lin}^{\etoe-\iter}) = O(d^2)
    \] 
\end{proposition}

\begin{proof}
    Trivial: The argument of \cite[Lemma 4.1]{joshi2025theory} shows that $|\cF_{d, \lin}| \leq 2^{O(d^2)}$. Since $|\cF_{d, \lin}^{\etoe-\iter}| \leq |\cF_{d, \lin}|$, the standard halving algorithm implies the stated bound.
\end{proof}

The bound above is based on a simple size argument. It is natural to ask if it is tight. Below, we give a positive answer to this question, that holds even for $\ct$-learning.

The idea is to implement a ``latch" mechanism (like in digital systems) that ``drags" the output bit of the base function throughout the entire $\ct$, and then just apply the best known bound for $\cF_{d, \lin}$.

\begin{theorem} \label{thm:linear-class-lower-bound}
    There exists an adversary forcing $\Omega(d^2)$ on any learner for $\cF_{d, \lin}^{\etoe-\iter}$ for all $d \geq 3$ and $\iter \geq 2$, even with $\ct$-supervision.
\end{theorem}

The key is the following lemma, that provides the latch mechanism, in the price of reducing the dimension by $2$.

\begin{lemma}[Latch lemma] \label{lem:latch}
    Fix $d \geq 3$ and denote $m:= d-2$. Then, for all $\boldsymbol{v} \in \mathbb{R}^m$ and $c \in \mathbb{R}$, there exists $f_{\boldsymbol{w},b} \in \cF_{d, \lin}$  such that
    \[
    f_{\boldsymbol{w},b}^{\etoe-\iter}(z 1 0) = f_{\boldsymbol{v},c}(z)
    \]
    for all $z \in \{0,1\}^m$ and all $\iter \geq 2$.
\end{lemma}

\begin{proof}
    Define
    \[
    L := \min_{z \in \{0,1\}^m} \langle \boldsymbol{v}, z \rangle + c, \quad U := \max_{z \in \{0,1\}^m} \langle \boldsymbol{v}, z \rangle + c
    \]
    and let $B := \max\{0, U\} +1, A:= B - L + 1$.
    We now define $f_{\boldsymbol{w},b} \in \cF_{d, \lin}$ that realizes the theorem.
    Let $s(x) := (u, p, q)$ be the last $d$ bits of an input $x \in \{0,1\}^\star$ for $f_{\boldsymbol{w},b}$, where $|u| = d-2, |p| = |q| =1$. Let
    \[
    \boldsymbol{w} := (\boldsymbol{v}, B, 2A), \quad b:= c- B.
    \]
    Therefore, for any input $x$ of length at least $d$ we have:
    \[
    f_{\boldsymbol{w},b}(x) = 1 \mleft[ \langle \boldsymbol{v}, u \rangle + B p + 2A q + c- B \geq 0 \mright].
    \]
    Now, we show that $f_{\boldsymbol{w},b}$ satisfies three desired properties:
    \begin{enumerate}
        \item \label{itm:first-correct} For all $z \in \{0,1\}^m$: $f_{\boldsymbol{w},b}(z 10) = f_{\boldsymbol{v},c}(z)$ (first  bit of $\ct$ is correct).
        \item \label{itm:drag-0} For any $x \in \{0,1\}^\star$ with $p = q = 0$, we have $f_{\boldsymbol{w},b}(x) = 0$ (always drag $0$).
        \item \label{itm:drag-1} For any $x \in \{0,1\}^\star$ with $q = 1$, we have $f_{\boldsymbol{w},b}(x) = 1$ (always drag $1$).
    \end{enumerate}
    The three items establish the statement of the lemma: The first item states that $f_{\boldsymbol{w},b}(z 10)$ returns the ``correct" bit $f_{\boldsymbol{v},c}(z)$. Then, since $f_{\boldsymbol{w},b}^{\etoe-\iter}$ appends the returned correct bit, precisely one of the following two states is reached: if $f_{\boldsymbol{v},c}(z) = 0$, then the last two bits become $0$, and the second item establishes that $0$ is dragged through all $\iter$ iterations. Otherwise, the last bit becomes $1$, and then the third item guarantees that $1$ is dragged through all $\iter$ iterations. So it remains to prove the three items.

    \paragraph{First item.}
    Let $z \in \{0,1\}^m$, and denote $x= z 1 0$. So
    \[
    \langle \boldsymbol{w}, s(x) \rangle + b =  \langle \boldsymbol{v}, z \rangle + B + c - B = \langle \boldsymbol{v}, z \rangle + c. 
    \]
    Therefore,
    \[
    f_{\boldsymbol{w},b}(x) = 1 \mleft [\langle \boldsymbol{w}, s(x) \rangle + b \geq 0  \mright]
    =
    1 \mleft [\langle \boldsymbol{v}, z \rangle + c \geq 0 \mright]
    =
    f_{\boldsymbol{v},c}(z).
    \]

    \paragraph{Second item.}
    Let $x$ such that $p=q=0$. Then
    \[
    \langle \boldsymbol{w}, s(x) \rangle + b =  \langle \boldsymbol{v}, u \rangle  + c - B \leq U-B < 0 , 
    \]
    where the first inequality is by definition of $U$, and the second inequality is by definition of $B$. Therefore, $f_{\boldsymbol{w},b}(x) = 0$ by definition.

    \paragraph{Third item.}
    Let $x$ such that $q=1$. Then
    \begin{align*}
    \langle \boldsymbol{w}, s(x) \rangle + b
    &=
    \langle \boldsymbol{v}, u \rangle + Bp +2A + c- B \\
    &\geq
    \langle \boldsymbol{v}, u \rangle + c +2A- B \\
    &\geq
    L +2A- B \\
    & = 
    L + 2(B - L + 1) - B \\
    &=
    B - L + 2 \\
    & >
    0.
    \end{align*}
    The equalities are by definitions. The first inequality is since $B > 0$ and thus $B p \geq 0$. The second inequality is by definition of $L$.
    The last inequality is since $B > U \geq L$.
    So, we have $f_{\boldsymbol{w},b}(x) = 1$ by definition.
\end{proof}

We may now prove Theorem~\ref{thm:linear-class-lower-bound}.
\begin{proof}[Proof of Theorem~\ref{thm:linear-class-lower-bound}]
    By \cite[Theorem B.2]{feldman2014sample}, we have $\LD(\cF_{m, \lin}) = \Theta(m^2)$ for all $m \geq 1$. As the latch lemma shows, for every function  $f_m \in \cF_{m, \lin}$ there is a function $f_d \in \cF_{d, \lin}$ where $d = m+2$, such that for all $z \in \{0,1\}^m$ and all $\iter \geq 1$ we have $f_d^{\etoe-\iter}(z 10) =  f_m(z)$. Therefore, the same adversarial strategy realizing the $\Omega(m^2)$ lower bound of \cite{feldman2014sample} can be implemented by $\cF_{d, \lin}^{\etoe-\iter}$, replacing any instance $z \in \{0,1\}^m$ with $z 10$, and any $f_m \in \cF_{m, \lin}$ with the appropriate $f_d \in \cF_{d, \lin}$. Therefore,
    \[
    \LD(\cF_{d, \lin}^{\etoe-\iter}) = \Theta(m^2) = \Theta(d^2)
    \]
    for all $\iter$. To see that the same lower bound holds even with $\ct$-supervision, note that the $\ct$ is just $\iter$ copies of the same bit, and so $\ct$-supervision is equivalent to only receiving the $\etoe$ output.
\end{proof}

The same technique allows to prove a lower bound for the statistical setting as well.
\begin{theorem}
    For all $d \geq 3$ and $\iter \geq 2$, any PAC-learner for $\cF_{d, \lin}^{\etoe-\iter}$ has sample complexity $\Omega(d)$, even with $\ct$-supervision.
\end{theorem}

\begin{proof}
    We prove that $\VC(\cF_{d, \lin}^{\etoe}) \geq m$ for $m:=d-2$. Since all functions we will use to shatter a set of size $m$ are implemented by the latch mechanism of Lemma~\ref{lem:latch}, those functions always copy the same final bit, and thus $\ct$-supervision is equivalent to only receiving the $\etoe$ output. Therefore, proving $\VC(\cF_{d, \lin}^{\etoe}) \geq m$ suffices for showing that $m$ is a lower bound on the sample complexity, even with $\ct$-supervision. The proof now follows essentially the same lines as the proof of Theorem~\ref{thm:linear-class-lower-bound}, only replacing the bound $\LD(\cF_{m, \lin}) = \Theta(m^2)$ used in the proof of Theorem~\ref{thm:linear-class-lower-bound} to the known bound $\VC(\cF_{m, \lin}) \geq m$. 
\end{proof}

\subsection{Bounds for Efficient learning} \label{sec:linear-computational}

We first show that efficient online $\etoe$-learning of autoregressive linear thresholds is impossible, conditioned on the following assumption; we refer the reader to, e.g., \cite{joshi2025theory} for more details. 

\begin{assumption}[Hardness of Learning $\mathsf{TC}^0$]
\label{ass:hardness-tc0}
There exist a constant $L>0$ and a polynomial $p(n)$ such that threshold circuits of depth $L$ and size $p(n)$ over $n$ binary inputs are not weakly PAC-learnable in $\operatorname{poly}(n)$ time, i.e., for some constants $\varepsilon<\tfrac12$ and $\delta<1$, there is no polynomial-time learner that achieves error at most $\varepsilon$ with confidence at least $1-\delta$.
\end{assumption}

Based on the above assumption, we show that learning end-to-end autoregressive linear classifiers is  impossible. The proof is based on a standard batch-to-online argument and using Theorem 4.4 in~\cite{joshi2025theory}. 

\begin{theorem}[Conditional hardness of efficient online $\etoe$-learning]
\label{thm:e2e-ltf-online-hard}
Assume Hardness Assumption~\ref{ass:hardness-tc0}. Then there is no possibly improper online learner such that, for every $d,\iter,n$, on every realizable online sequence of examples from $\{0,1\}^n$ labeled by some target in $\cF_{d,\lin}^{\etoe-\iter}$, the learner runs in time polynomial in $d,\iter,n$ per round and makes at most polynomially many mistakes in $d,\iter,n$.
\end{theorem}

\begin{proof}
Suppose such a learner exists. By the standard reduction from realizable online learning with polynomial mistake bound to realizable PAC learning, this yields a polynomial-time realizable PAC learner for $\cF_{d,\lin}^{\etoe-\iter}$ over $\{0,1\}^n$. This contradicts \cite[Theorem~4.4]{joshi2025theory}.
\end{proof}

In contrast, with $\ct$-supervision there is an efficient online learner matching the information-theoretic guarantee up to a logarithmic factor. We use the classical deterministic online learner for Boolean linear thresholds. We first give a reduction 

\begin{lemma}[Online reduction with $\ct$-supervision]
\label{lem:online-cot-reduction}
Let $\cF \subseteq \{f:\{0,1\}^\star \to \{0,1\}\}$ be a base class. Assume there is a deterministic online learner $\cA$ for $\cF$ such that:
\begin{enumerate}
    \item on every realizable sequence it makes at most $K$ mistakes;
    \item on every length-$L$ sequence, its total computation time is polynomial in $L$ and the total input length.
\end{enumerate}
Then for every $\iter\ge 1$ there is a deterministic online learner $\cB$ for $\cF^{\etoe-\iter}$ with $\ct$-supervision such that:
\begin{enumerate}
    \item on every realizable $\ct$-supervised sequence it makes at most $K$ mistakes;
    \item on every length-$n$ sequence, its total computation time is polynomial in $n,\iter,$ and the total input length.
\end{enumerate}
\end{lemma}

\begin{proof}
Fix a realizable $\ct$-supervised sequence $(x_t,z_t)_{t=1}^n$, where
$z_t=(z_{t,1},\ldots,z_{t,\iter})=f_\star^{\ct-\iter}(x_t)$ for some $f_\star\in\cF$.
For each round $s$, let $u_{s,0}:=x_s$, and for $j\in[\iter]$ let
$u_{s,j}:=u_{s,j-1}z_{s,j}$, where for a binary string $u$ and a bit
$b\in\{0,1\}$, $ub$ denotes the string obtained by appending $b$ to $u$.
Learner $\cB$ maintains the ordered transcript
\[
S_{t-1}:=\bigl((u_{1,0},z_{1,1}),\ldots,(u_{1,\iter-1},z_{1,\iter}),
\ldots,
(u_{t-1,0},z_{t-1,1}),\ldots,(u_{t-1,\iter-1},z_{t-1,\iter})\bigr).
\]

On round $t$, learner $\cB$ sets $\hat u_{t,0}:=x_t$. For $j=1,\ldots,\iter$, suppose $\hat z_{t,1},\ldots,\hat z_{t,j-1}$ have already been defined and $\hat u_{t,j-1}$ has been formed. Learner $\cB$ runs $\cA$ from scratch on the sequence
\[
S_{t-1},(\hat u_{t,0},\hat z_{t,1}),\ldots,(\hat u_{t,j-2},\hat z_{t,j-1}),
\]
where for $j=1$ this means just $S_{t-1}$, and lets $\hat z_{t,j}$ be the prediction on the next example $\hat u_{t,j-1}$; then it sets $\hat u_{t,j}:=\hat u_{t,j-1}\hat z_{t,j}$. Finally, it predicts $\hat z_{t,\iter}$.
After receiving the true chain-of-thought $z_t=(z_{t,1},\ldots,z_{t,\iter})$, it updates the transcript to
\[
S_t:=\bigl(S_{t-1},(u_{t,0},z_{t,1}),\ldots,(u_{t,\iter-1},z_{t,\iter})\bigr).
\]

Now consider the expanded sequence
\[
((u_{t,j-1},z_{t,j}))_{1\le t\le n,\ 1\le j\le \iter}.
\]
It is realizable by $f_\star$, since $z_{t,j}=f_\star(u_{t,j-1})$ for all $t,j$.
If $\cB$ makes a mistake on round $t$, let $j$ be the first index with $\hat z_{t,j}\neq z_{t,j}$.
Then $\hat u_{t,i}=u_{t,i}$ for all $i<j$, so the transcript used by $\cB$ to compute $\hat z_{t,j}$ is exactly the prefix of the expanded realizable sequence seen by $\cA$ just before processing $(u_{t,j-1},z_{t,j})$. Hence $\cA$ makes a mistake on that example. Therefore every final-answer mistake of $\cB$ is charged to a distinct mistake of $\cA$, and so $\cB$ makes at most $K$ mistakes.

Finally, the expanded sequence has length at most $n\iter$, and its total input length is
\[
\sum_{t=1}^n \sum_{j=1}^\iter |u_{t,j-1}|
=
\sum_{t=1}^n \sum_{j=1}^\iter (|x_t|+j-1)
\le
\iter \sum_{t=1}^n |x_t| + n\iter^2,
\]
which is polynomial in $n,\iter,$ and the total input length of the original sequence. The same bound holds for every simulated prefix sequence used by $\cB$, since $|\hat u_{t,j}|=|x_t|+j$ for all $t,j$. Since $\cB$ performs at most $n\iter$ simulations, and each simulation runs $\cA$ from scratch on a sequence of length at most $n\iter$ and polynomially bounded total input length, the total computation time of $\cB$ is polynomial in $n,\iter,$ and the total input length of the original sequence.
\end{proof}

To obtain our efficient algorithm, we use the following classic result for the base class.

\begin{lemma}[Maass--Tur\'an \cite{MaassTuran1994}]
\label{lem:online-boolean-ltf}
There exists a deterministic online learner for Boolean linear thresholds over $\{0,1\}^d$ such that:
\begin{enumerate}
    \item on every realizable sequence it makes at most $O(d^2\log d)$ mistakes;
    \item on every length-$m$ sequence, its total computation time is polynomial in $m$ and $d$.
\end{enumerate}
\end{lemma}

\begin{theorem}[Efficient online $\ct$-learning for autoregressive linear thresholds]
\label{thm:cot-ltf-online-efficient}
Fix a horizon $n$. There exists a deterministic online learner for $\cF_{d,\lin}^{\etoe-\iter}$ with $\ct$-supervision such that, for every sequence $(x_t,z_t)_{t=1}^n$ satisfying $z_t=f_\star^{\ct-\iter}(x_t)$ for some $f_\star\in\cF_{d,\lin}$, the learner makes at most $O(d^2\log d)$ mistakes and its total computation time is polynomial in $n,d,\iter,$ and $\sum_{t=1}^n |x_t|$.
\end{theorem}

\begin{proof}
Let $\phi:\{0,1\}^\star\to\{0,1\}^d$ be the suffix map returning the last $d$ bits, padded with leading zeroes if needed. By definition of $\cF_{d,\lin}$, for every $f\in\cF_{d,\lin}$ there is a Boolean linear threshold $g_f:\{0,1\}^d\to\{0,1\}$ such that $f(x)=g_f(\phi(x))$ for all $x\in\{0,1\}^\star$.

Let $\cA$ be the learner from Lemma~\ref{lem:online-boolean-ltf}. We define a learner $\cA'$ for $\cF_{d,\lin}$ as follows: on an example $x$, learner $\cA'$ predicts whatever $\cA$ predicts on $\phi(x)$; after receiving the label $y$, learner $\cA'$ updates $\cA$ on the labeled example $(\phi(x),y)$. Since $f(x)=g_f(\phi(x))$, every realizable sequence for $\cF_{d,\lin}$ maps to a realizable sequence for Boolean linear thresholds over $\{0,1\}^d$. Hence $\cA'$ makes at most $O(d^2\log d)$ mistakes on every realizable sequence. Its total computation time on any length-$m$ sequence is polynomial in $m,d,$ and the total input length, since computing all suffixes $\phi(x)$ contributes only linear overhead.

Applying Lemma~\ref{lem:online-cot-reduction} to the base class $\cF_{d,\lin}$ and the learner $\cA'$ yields the claim.
\end{proof}

\section{Non-Littlestone base classes} \label{sec:non-lit}

\begin{theorem}\label{thm:alternating-horizons}
There exists a base class \(\cF\) over the binary alphabet \(\Sigma=\{0,1\}\) such that, for the pairwise distinct sequence
\[
\iter_i := i+4 \qquad (i\in \N),
\]
the following holds:
\[
i \text{ odd } \implies \LD(\cF^{\etoe-\iter_i}) = 0,
\qquad
i \text{ even } \implies \LD(\cF^{\etoe-\iter_i})=\infty.
\]
Moreover, for every odd \(i\), online learning with no \(\ct\)-supervision is trivial, while for every even \(i\), online learning with \(\ct\)-supervision is impossible.
\end{theorem}

\begin{proof}
For every \(m\ge 2\) and \(n\ge 0\), define the string
\[
u_{m,n} := 0^m 1 0^n 1.
\]
For every binary array \(\alpha=(\alpha_{m,n})_{m\ge 2,\; n\ge 0}\in \{0,1\}^{\{2,3,\ldots\}\times (\N\cup\{0\})}\), define a next-token generator
\[
f_\alpha:\{0,1\}^\star \to \{0,1\}
\]
as follows:
\[
f_\alpha(x)=
\begin{cases}
\alpha_{m,n}
& \text{if } x=u_{m,n} \text{ for some } m\ge 2,\ n\ge 0,\\[1mm]
0
& \text{if } x=u_{m,n} b 0^t \text{ for some } m\ge 2,\ n\ge 0,\ b\in\{0,1\},\ 0\le t<2m-2,\\[1mm]
b
& \text{if } x=u_{m,n} b 0^{2m-2} \text{ for some } m\ge 2,\ n\ge 0,\ b\in\{0,1\},\\[1mm]
0
& \text{otherwise.}
\end{cases}
\]
Let
\[
\cF := \{f_\alpha : \alpha\in \{0,1\}^{\{2,3,\ldots\}\times (\N\cup\{0\})}\}.
\]

We first analyze the even horizons. Fix \(m\ge 2\), \(n\ge 0\), and \(\alpha\). Let \(b:=\alpha_{m,n}\). Starting from the input \(u_{m,n}\), the generated bits are
\[
b,\underbrace{0,\ldots,0}_{2m-2\text{ times}},b.
\]
Therefore,
\[
f_\alpha^{\etoe-2m}(u_{m,n})=\alpha_{m,n}.
\]
Hence, for every fixed \(m\ge 2\), the set
\[
X_m:=\{u_{m,n}:n\ge 0\}
\]
shows that \(\cF^{\etoe-2m}\) has infinite Littlestone dimension. Indeed, let \(d\in\N\), and consider the complete binary tree of depth \(d\) in which every node at depth \(j\) is labeled by \(u_{m,j-1}\). Given any branch-label sequence \(y_1,\ldots,y_d\in\{0,1\}\), choose \(\alpha\) so that
\[
\alpha_{m,j-1}=y_j \qquad \text{for all } j=1,\ldots,d,
\]
and define the remaining coordinates of \(\alpha\) arbitrarily. Then
\[
f_\alpha^{\etoe-2m}(u_{m,j-1}) = y_j
\qquad\text{for all } j=1,\ldots,d.
\]
Thus this tree is shattered by \(\cF^{\etoe-2m}\). Since \(d\) is arbitrary,
\[
\LD(\cF^{\etoe-2m})=\infty
\qquad\text{for every } m\ge 2.
\]

Moreover, this lower bound is unchanged under \(\ct\)-supervision. Indeed, for every queried instance \(u_{m,n}\), if \(b:=\alpha_{m,n}\), then the revealed \(2m\)-step chain is
\[
f_\alpha^{\ct-2m}(u_{m,n}) = b0^{2m-2}b.
\]
In particular, the revealed \(\ct\) is completely determined by the final label
\[
f_\alpha^{\etoe-2m}(u_{m,n}) = b.
\]
Hence, on the adversarial sequence above, a learner with \(\ct\)-supervision receives no more information than an \(\etoe\)-learner, so the same lower bound applies.

We now turn to odd horizons. Fix an odd \(\iter\ge 5\). We claim that all functions in \(\cF^{\etoe-\iter}\) coincide, and therefore \(\LD(\cF^{\etoe-\iter})=0\).

Fix \(x\in\{0,1\}^\star\) and two arrays \(\alpha,\beta\). We show that
\[
f_\alpha^{\etoe-\iter}(x)=f_\beta^{\etoe-\iter}(x).
\]

We will use the following simple observation: for every \(m\ge 2\), \(n\ge 0\), \(b\in\{0,1\}\), and \(r\ge 0\), the string
\[
u_{m,n} b 0^{2m-2} b 0^r
\]
is of none of the special forms appearing in the definition of \(f_\alpha\). Indeed, if it were of the form \(u_{m',n'}\) or \(u_{m',n'} c0^t\), then necessarily \(m'=m\) and \(n'=n\), since the first two occurrences of \(1\) already determine the prefix \(u_{m,n}\). But then the remaining suffix \(b0^{2m-2}b0^r\) cannot be empty, and also cannot equal \(c0^t\): if \(b=1\) then it contains two \(1\)'s, while if \(b=0\) it equals \(0^{2m+r}\), which would force \(c=0\) and \(t=2m+r-1>2m-2\). Thus no such string is special.

We now distinguish three cases.

\paragraph{Case 1: \(x=u_{m,n}\) for some \(m\ge 2\), \(n\ge 0\).}
Let \(b:=\alpha_{m,n}\). Under \(f_\alpha\), the first generated bit is \(b\), the next \(2m-2\) generated bits are \(0\), and the \(2m\)-th generated bit is again \(b\). After these \(2m\) steps, the current string is
\[
u_{m,n} b 0^{2m-2} b.
\]
By the observation above, from this point on, only the default rule applies, so all subsequent generated bits are \(0\). Therefore the only possibly \(\alpha\)-dependent generated bits occur at steps \(1\) and \(2m\). Since \(\iter\ge 5\) is odd it holds that \(\iter\neq 1\) and \(\iter\neq 2m\), and hence \(f_\alpha^{\etoe-\iter}(x)\) is independent of \(\alpha\). In particular,
\[
f_\alpha^{\etoe-\iter}(x)=f_\beta^{\etoe-\iter}(x).
\]

\paragraph{Case 2: \(x=u_{m,n} b 0^t\) for some \(m\ge 2\), \(n\ge 0\), \(b\in\{0,1\}\), \(0\le t\le 2m-2\).}
Starting from \(x\), the next \(2m-2-t\) generated bits are \(0\), then the next generated bit is \(b\), and after that the current string has the form
\[
u_{m,n} b 0^{2m-2} b.
\]
By the observation above, from this point on, all subsequent generated bits are \(0\). Therefore, the entire future trajectory from \(x\) is determined by the string \(x\) itself, and in particular is independent of \(\alpha\). Hence
\[
f_\alpha^{\etoe-\iter}(x)=f_\beta^{\etoe-\iter}(x).
\]

\paragraph{Case 3: \(x\) is of neither of the above two forms.}
Then \(f_\alpha(x)=f_\beta(x)=0\). As long as the autoregressive trajectory stays outside the set
\[
\{u_{m,n}:m\ge 2,\ n\ge 0\}
\;\cup\;
\{u_{m,n} b0^t : m\ge 2,\ n\ge 0,\ b\in\{0,1\},\ 0\le t\le 2m-2\},
\]
both functions continue to output \(0\). If the trajectory never enters this set, then clearly
\[
f_\alpha^{\etoe-\iter}(x)=f_\beta^{\etoe-\iter}(x)=0.
\]
Otherwise, let \(s\ge 1\) be the first time such a string is reached. Since all bits generated before time \(s\) are \(0\), the reached string cannot be of the form \(u_{m,n}\), because every \(u_{m,n}\) ends with \(1\). Hence at time \(s\) the trajectory reaches a string covered by Case 2, and from that point on the continuation is independent of \(\alpha\). Therefore
\[
f_\alpha^{\etoe-\iter}(x)=f_\beta^{\etoe-\iter}(x).
\]

We have proved that for every odd \(\iter\ge 5\), all functions in \(\cF^{\etoe-\iter}\) are identical. Hence
\[
|\cF^{\etoe-\iter}|=1,
\qquad\text{so}\qquad
\LD(\cF^{\etoe-\iter})=0.
\]
Therefore, for every odd \(\iter\ge 5\), online learning with \(\ct\)-supervision is also possible, since \(\ct\)-supervision is only more informative than \(\etoe\)-supervision.

Finally, for the pairwise distinct sequence \(\iter_i=i+4\), if \(i\) is even then \(\iter_i\) is even and at least \(6\), so
\[
\LD(\cF^{\etoe-\iter_i})=\infty,
\]
and moreover, the same impossibility holds even for online learning with \(\ct\)-supervision. If \(i\) is odd then \(\iter_i\) is odd and at least \(5\), so
\[
\LD(\cF^{\etoe-\iter_i})=0<\infty,
\]
and therefore online learning with \(\ct\)-supervision is also possible.
\end{proof}

	\section{Stochastic Autoregressive Lower Bound} \label{sec:stochastic}
	
	We now consider a stochastic variant of the autoregressive model.
	
	\begin{definition}[Stochastic autoregressive next-token generators]
		A \emph{stochastic base class} is a class
		\[
		\mathcal{G} \subseteq \{g:\{0,1\}^\star \to \Delta(\{0,1\})\},
		\]
		where $\Delta(\{0,1\})$ denotes the set of distributions over $\{0,1\}$. For $g \in \mathcal{G}$ and $x\in\{0,1\}^\star$, define $\bar g(x)$ to be the random string obtained by sampling $Y \sim g(x)$ and outputting $\operatorname{append}(x,Y)$. For $M \in \N$, define $g^{\ct-M}$ and $g^{\etoe-M}$ exactly as in the deterministic case, except that these are now induced distributions obtained by composing $\bar g$ for $M$ many steps. As before, write $\mathcal{G}^{\etoe-M} := \{g^{\etoe-M} : g \in \mathcal{G}\}$.
	\end{definition}
	
	In the associated online prediction problem, on round $t$ the learner receives an instance $x_t \in \{0,1\}^\star$, predicts a bit $\hat y_t \in \{0,1\}$, and then observes a random label $Y_t$, sampled from the relevant target distribution. In the direct next-token problem we have $Y_t \sim g_\star(x_t)$, while in the $\etoe$ problem we have $Y_t \sim g_\star^{\etoe-M}(x_t)$, for some unknown $g_\star \in \mathcal{G}$. For a learner $\cA$, target $g_\star$, and instance sequence $(x_t)_{t=1}^T$, define the expected regret relative to the optimal prediction 
	\[
	\operatorname{Reg}_T(\cA,g_\star,(x_t))
	:=
	\Ex\!\left[\sum_{t=1}^T \ind[\hat y_t \neq Y_t]\right]
	-
	\sum_{t=1}^T \min_{a \in \{0,1\}} \Pr[a \neq Y_t].
	\]
	
	\begin{theorem}
		\label{thm:stochastic-e2e-separation}
		For every odd $M \geq 1$ there exists a stochastic base class $\mathcal{G} = \{g_{-},g_{+}\}$ such that:
		\begin{enumerate}
			\item there exists an online learner for the direct next-token problem on $\mathcal{G}$ whose expected regret is $O(1)$ on every instance sequence;
			\item for every (possibly randomized) online learner $\cA$ for $\mathcal{G}^{\etoe-M}$, there exists a target $g_\star \in \mathcal{G}$ such that, even on the constant sequence $x_t = \emptyset$, its expected regret is $\Omega(2^M)$ at some horizon $T = \Theta(4^M)$.
		\end{enumerate}
	\end{theorem} In particular, the $\etoe$ regret can be $\Omega\left(2^M\right)$, much larger than $\log M$ which is the upper bound growth of $\etoe$-learning with a fixed Littlestone dimension base class (given in Theorem~\ref{thm:e2e-upper-bound}) in the standard (deterministic functions) model. 
	
	\begin{proof}
		For $x \in \{0,1\}^\star$, let $b(x)$ denote the last bit of $x$, with the convention $b(\emptyset)=0$. For $\sigma \in \{-1,+1\}$ define
		\[
		g_\sigma(x)
		=
		\text{law of } b(x) \oplus Z,
		\qquad
		Z \sim \Ber\!\left(\frac12 + \frac{\sigma}{4}\right).
		\]
		Thus $g_{-}$ flips the current last bit with probability $1/4$, whereas $g_{+}$ flips it with probability $3/4$.
		
		We first prove the first item in the theorem, that learning the direct next token is trivial. Fix a target $g_\sigma$ and an arbitrary instance sequence $(x_t)_{t=1}^T$. Define $Z_t := Y_t \oplus b(x_t)$. Then $Z_1,Z_2,\ldots$ are i.i.d.\ $\Ber(\frac12+\frac{\sigma}{4})$, independently of the instance sequence. Consider the learner that predicts arbitrarily on round $1$, and on round $t \geq 2$ computes the empirical mean $\bar Z_{t-1} := \frac{1}{t-1}\sum_{s=1}^{t-1} Z_s$, sets $\hat\sigma_t = +1$ if $\bar Z_{t-1} \geq \frac12$ and $\hat\sigma_t=-1$ otherwise, and then predicts $b(x_t)$ if $\hat\sigma_t=-1$ and $1-b(x_t)$ if $\hat\sigma_t=+1$. If $\hat\sigma_t=\sigma$, this is the Bayes-optimal prediction for $g_\sigma(x_t)$; otherwise the regret at round $t$ is exactly $3/4-1/4=1/2$. Hence,
		\[
		\operatorname{Reg}_T
		\leq
		\frac12 + \frac12 \sum_{t=2}^T \Pr[\hat\sigma_t \neq \sigma].
		\]
		Since the true mean is either $1/4$ or $3/4$, Hoeffding's inequality gives $\Pr[\hat\sigma_t \neq \sigma] \leq e^{-(t-1)/8}$, and therefore $\operatorname{Reg}_T = O(1)$.
		
	We now prove the second item. If $M=1$, then $g_\sigma^{\etoe-1}=g_\sigma$. On the constant sequence $x_t=\emptyset$, the two targets induce $\Ber(1/4)$ and $\Ber(3/4)$. On round $1$, if the learner predicts $1$ with probability $p$, then its expected regret is $p/2$ under $g_{-}$ and $(1-p)/2$ under $g_{+}$. Thus one of the two targets yields regret at least $1/4$ at horizon $T=1$, which is $\Omega(2^M)$ for $M=1$. So it remains to consider odd $M \geq 3$.
	
	Fix such an odd $M$, and again consider the constant sequence $x_t=\emptyset$. Let $q_m^\sigma$ be the probability that after $m$ autoregressive steps the current last bit equals $1$ under $g_\sigma$, starting from $\emptyset$. The process is a two-state Markov chain, and
	\[
	q_{m+1}^\sigma
	=
	\left(\frac12+\frac{\sigma}{4}\right)
	+
	\left(1-2\left(\frac12+\frac{\sigma}{4}\right)\right) q_m^\sigma,
	\qquad
	q_0^\sigma = 0.
	\]
	Solving this recursion gives $q_m^\sigma = \frac{1-(-\sigma/2)^m}{2}$. Since $M$ is odd, $q_M^\sigma = \frac12 + \sigma 2^{-M-1}$. Let $\delta_M := 2^{-M-1}$. Hence on input $\emptyset$, the final answer under $g_{-}^{\etoe-M}$ is $\Ber(\frac12-\delta_M)$, while under $g_{+}^{\etoe-M}$ it is $\Ber(\frac12+\delta_M)$. The Bayes-optimal predictions are therefore $0$ and $1$ respectively, and using the wrong one incurs regret $2\delta_M$ per round.
	
	Fix any randomized learner $\cA$. Its prior-averaged expected regret under the uniform prior on $\{g_{-},g_{+}\}$ is the expectation, over the learner's internal randomness, of the corresponding quantity for the induced deterministic learner. It therefore suffices to prove the lower bound for deterministic learners.
	
Fix such a deterministic learner. On round $t$, its prediction is a function of the first $t-1$ observed labels. Let $P_-^{t-1}$ and $P_+^{t-1}$ denote the distributions of these $t-1$ labels under $\Ber(\frac12-\delta_M)$ and $\Ber(\frac12+\delta_M)$ respectively, and let $A_t$ be the event that the learner predicts $1$ on round $t$. Since the Bayes-optimal prediction is $0$ under $g_{-}$ and $1$ under $g_{+}$, the prior-averaged probability of predicting the wrong Bayes-optimal bit at round $t$ is
\[
\frac12 P_-^{t-1}(A_t) + \frac12 P_+^{t-1}(A_t^c).
\]
By the Bretagnolle--Huber inequality \cite[Theorem 14.2]{lattimore2020bandit},
\[
P_-^{t-1}(A_t) + P_+^{t-1}(A_t^c)
\geq
\frac12 \exp\!\left(-\mathrm{KL}(P_-^{t-1},P_+^{t-1})\right),
\]
where $\mathrm{KL}(P,Q)$ denotes the Kullback--Leibler divergence from $P$ to $Q$. Therefore the prior-averaged probability of predicting the wrong Bayes-optimal bit is at least
\[
\frac14 \exp\!\left(-\mathrm{KL}(P_-^{t-1},P_+^{t-1})\right).
\]

It remains to bound the divergence. Since the labels are i.i.d.,
\[
\mathrm{KL}(P_-^{t-1},P_+^{t-1})
=
(t-1)\,\mathrm{KL}\!\left(\Ber\!\left(\frac12-\delta_M\right),\Ber\!\left(\frac12+\delta_M\right)\right).
\]
Moreover,
\[
\mathrm{KL}\!\left(\Ber\!\left(\frac12-\delta\right), \Ber\!\left(\frac12+\delta\right)\right)
=
2\delta \log\!\left(\frac{1+2\delta}{1-2\delta}\right)
\leq
16\delta^2
\qquad
(\delta \leq 1/4),
\]
where the inequality uses $\log\!\left(\frac{1+u}{1-u}\right)\le 4u$ for $u \in [0,1/2]$. Hence
\[
\mathrm{KL}(P_-^{t-1},P_+^{t-1}) \leq 16(t-1)\delta_M^2.
\]
Thus, whenever $t \leq 1/(32\delta_M^2)$, we have $\mathrm{KL}(P_-^{t-1},P_+^{t-1}) \leq 1/2$, and so the prior-averaged probability of predicting the wrong Bayes-optimal bit is at least $\frac14 e^{-1/2}=\Omega(1)$. Since each such mistake contributes regret $2\delta_M$, the prior-averaged regret at round $t$ is $\Omega(\delta_M)$.
	
	Setting $T := \left\lfloor \frac{1}{32\delta_M^2} \right\rfloor$, we have $T=\Theta(4^M)$. Summing over rounds $t=1,\ldots,T$, the prior-averaged regret up to time $T$ is $\Omega(T\delta_M)=\Omega(1/\delta_M)=\Omega(2^M)$. Therefore at least one of the two targets $g_{-},g_{+}$ yields regret $\Omega(2^M)$ at horizon $T=\Theta(4^M)$.
	\end{proof}

\section*{Acknowledgments}
We thank Steve Hanneke and Shay Moran for their significant contribution to this work, and concretely for suggesting the class and the lower bound proof strategy in Theorem~\ref{thm:e2e-taxonomy-intro}.

We also thank anonymous NeurIPS 2026 reviewers for their valuable suggestions, which we think have helped to improve the paper. In particular, one of the reviewers noticed that Theorem~\ref{thm:pac-cot-bound-lit-intro} can alternatively be proved via a result of \cite{hanneke2026sample}.

Idan Mehalel is supported by the European Research Council (ERC) under the European Union’s Horizon 2022 research and innovation program (grant agreement No. 101041711), the Israel Science Foundation (grant number 2258/19), and the Simons Foundation (as part of the Collaboration on the Mathematical and Scientific Foundations of Deep Learning).

Ilan Doron-Arad is supported by grant NSF DMS-2031883 and Vannevar Bush Faculty Fellowship ONR-N00014-20-1-2826 (PI Mossel).

Elchanan Mossel is partially supported by NSF DMS-2031883, Vannevar Bush Faculty Fellowship ONR-N00014-20-1-2826, MURI N000142412742, and a Simons Investigator Award.

\bibliographystyle{alphaurl}
\bibliography{bib.bib}

\end{document}